%% file: 0_main.tex
\documentclass{article}
\pdfoutput=1

\usepackage{microtype}
\usepackage{graphicx}
\usepackage{subfigure}
\usepackage{booktabs} 
\usepackage{multicol}

\usepackage{hyperref}

\usepackage[accepted]{icml2023}

\usepackage{amsmath}
\usepackage{amssymb}
\usepackage{mathtools}
\usepackage{amsthm}
\usepackage{bbm}

\usepackage[capitalize,noabbrev]{cleveref}

\theoremstyle{plain}
\newtheorem{theorem}{Theorem}[section]

\newtheorem{lemma}[theorem]{Lemma}
\newtheorem{corollary}[theorem]{Corollary}
\theoremstyle{definition}

\newtheorem{assumption}[theorem]{Assumption}
\theoremstyle{remark}

\usepackage[textsize=tiny]{todonotes}

\icmltitlerunning{Optimistic Planning by Regularized Dynamic Programming}

\input{gmacros}

\usepackage{xspace}

\usepackage{notations}
\newcommand{\MainAlg}{\textbf{\texttt{RAVI-UCB}}\xspace}
\newcommand{\PlanningAlg}{\MainAlg}

\renewcommand{\phi}{\varphi}
\renewcommand{\epsilon}{\varepsilon}

\begin{document}

\twocolumn[
\icmltitle{Optimistic Planning by Regularized Dynamic Programming}



\icmlsetsymbol{equal}{*}

\begin{icmlauthorlist}
\icmlauthor{Antoine Moulin}{upf}
\icmlauthor{Gergely Neu}{upf}
\end{icmlauthorlist}

\icmlaffiliation{upf}{Universitat Pompeu Fabra, Barcelona, Spain}

\icmlcorrespondingauthor{Antoine Moulin}{antoine.moulin@upf.edu}
\icmlcorrespondingauthor{Gergely Neu}{gergely.neu@gmail.com}

\icmlkeywords{Machine Learning, ICML}

\vskip 0.3in
]



\printAffiliationsAndNotice{}  

\begin{abstract}
    We propose a new method for optimistic planning in infinite-horizon discounted Markov decision processes based on the idea of adding regularization to the updates of an otherwise standard approximate value iteration procedure. This technique allows us to avoid contraction and monotonicity arguments typically required by existing analyses of approximate dynamic programming methods, and in particular to use approximate transition functions estimated via least-squares procedures in MDPs with linear function approximation. We use our method to recover known guarantees in tabular MDPs and to provide a computationally efficient algorithm for learning near-optimal policies in discounted linear mixture MDPs from a single stream of experience, and show it achieves near-optimal statistical guarantees.
\end{abstract}

\input{1_introduction.tex}
\input{2_preliminaries.tex}
\input{3_algorithm.tex}
\input{4_analysis.tex}
\input{5_applications.tex}
\input{6_conclusion.tex}
\input{7_acknowledgements.tex}

\bibliography{ref}
\bibliographystyle{icml2023antoine}

\newpage
\appendix
\onecolumn
\input{A_main_proofs.tex}
\newpage
\input{C_algorithms.tex}
\newpage
\input{B_standard_results.tex}


\end{document}

%% file: gmacros.tex
\newcommand{\CB}{\textup{CB}}

\newcommand{\piout}{\pi_{\text{out}}}
\newcommand{\reset}{reset\xspace}

\newcommand{\dd}{\mathrm{d}}

\newcommand{\F}{\mathcal{F}}

\newcommand{\real}{\mathbb{R}}

\newcommand{\trace}[1]{\mbox{tr}\left(#1\right)}
\newcommand{\II}[1]{\mathbb{I}_{\left\{#1\right\}}}
\newcommand{\PP}[1]{\mathbb{P}\left[#1\right]}

\newcommand{\EE}[1]{\mathbb{E}\left[#1\right]}

\newcommand{\EEs}[2]{\mathbb{E}_{#2}\left[#1\right]}

\newcommand{\EEcc}[2]{\mathbb{E}\left[\left.#1\right|#2\right]}

\newcommand{\ra}{\rightarrow}

\newcommand{\babs}[1]{\bigl| #1 \bigr|}

\newcommand{\iprod}[2]{\left\langle#1,#2\right\rangle}
\newcommand{\biprod}[2]{\bigl\langle#1,#2\bigr\rangle}

\newcommand{\norm}[1]{\left\|#1\right\|}

\newcommand{\twonorm}[1]{\norm{#1}_2}

\newcommand{\ev}[1]{\left\{#1\right\}}
\newcommand{\pa}[1]{\left(#1\right)}
\newcommand{\bpa}[1]{\bigl(#1\bigr)}

\newcommand{\wh}{\widehat}
\newcommand{\wt}{\widetilde}

\newcommand{\htheta}{\wh{\theta}}

\newcommand{\hP}{\wh{P}}

\newcommand{\transpose}{^\mathsf{\scriptscriptstyle T}}

\usepackage{todonotes}
\definecolor{PalePurp}{rgb}{0.66,0.57,0.66}

%% file: 1_introduction.tex

\section{Introduction}\label{sec:introduction}
The idea of constructing a confidence set of statistically plausible models and picking a policy that maximizes the expected return in the best of these models can be traced back to the pioneering work of \citet{LR85} in the context of multi-armed bandit problems, and has been successfully extended to address the exploration-exploitation dilemma in reinforcement learning (RL, \citealp{SB18}). This popular design principle came to be known as \emph{optimism in the face of uncertainty}, and the associated optimization task as the problem of \emph{optimistic planning}. The optimistic principle has driven the development of statistically efficient RL algorithms for a variety of problem settings. Following the work of \citet{BT02,SLL09} on optimistic exploration methods for RL in Markov decision processes (MDPs), a breakthrough was achieved by \citet*{jaksch2010near}, whose UCRL2 algorithm was shown to achieve near-optimal regret guarantees in a broad class of tabular MDPs. In subsequent years, their work inspired an impressive amount of follow-up work, leading to a variety of extensions, improvements, and other mutations.

The computational efficiency of such optimistic methods crucially hinges on the implementation of the optimistic planning subroutine. In the work of \citet{jaksch2010near}, this was addressed by a procedure called \emph{extended value iteration} (EVI), which performs dynamic programming (DP) in an auxiliary MDP where the confidence set of models is projected to the space of actions, allowing the realization of arbitrary transitions that are statistically plausible given all past experience. After mild adjustments, the EVI procedure can be shown to give near-optimal solutions to the optimistic planning problem in a computationally efficient manner (cf.~\citealp{fruit2018efficient} and Section~38.5.2 in \citealp{LSz20}). Other, even more effective optimistic dynamic programming procedures have been proposed and analyzed \citep{fruit2018efficient,QFPL18}. However, these computational developments have been largely restricted to the relatively simple tabular setting.

In recent years, the RL theory literature has seen a massive revival largely due to the breakthrough achieved by \citet*{jin2020provably}, who successfully extended the idea of optimistic exploration to a class of large-scale MDPs using linear function approximation. While extremely influential, their approach (and virtually all of its numerous follow-ups) are limited to the relatively simple setting of \emph{finite-horizon MDPs}. The reason for this limitation is inherent in their algorithm design that crucially uses the fact that optimistic planning in finite-horizon MDPs can be solved via a simple backward recursion over the time indices within each episode \citep{neu2020unifying}. This idea completely fails for infinite-horizon problems where dynamic programming methods should aim to approximate the solution of \emph{a fixed-point equation}. Solving such fixed-point equations is possible in the tabular case but no known efficient method exists for linear function approximation, the short reason being that the least-squares transition estimator used in the construction of \citet{jin2020provably} cannot be straightforwardly used to build an approximate Bellman operator that satisfies the necessary contraction properties.

The best attempt at tackling the infinite-horizon setting under function approximation we are aware of is by \citet*{wei2021learning}, who propose algorithms that are either statistically or computationally efficient, but fall short of providing an algorithm with both of these desired properties. Another good contribution was made by \citet*{VPSS22}, who provided approximate DP methods for stochastic shortest path problems with linear transition functions, and analyzed them via studying the concentration properties of the empirical transition operator. This technique allowed them to prove regret bounds, but the guarantees did not reach optimality in terms of scaling with the time horizon unless strong assumptions are made. Notably, \citet{VPSS22} only managed to perform a tight analysis in the special case where the features are orthogonal, which allowed them to reason about contraction properties of the empirical Bellman operator. Lacking a general contraction argument, or another idea that would enable computationally efficient optimistic planning, efficient exploration-exploitation in infinite-horizon MDPs under function approximation has remained unsolved so far.\looseness=-1

This is the problem we address in this paper in the context of \emph{discounted} infinite-horizon MDPs. Instead of relying on a contraction argument (or an approximate version thereof), we propose to solve the optimistic planning problem using \emph{regularized dynamic programming}. In particular, we consider a variant of the Mirror-Descent Modified Policy Iteration (MD-MPI) algorithm of \citet*{geist2019theory} that uses a least-squares estimator of the transition kernel and an exploration bonus to define an optimistic regularized Bellman operator. Using arguments from the classic analysis of mirror descent methods, we show that each application of this optimistic operator improves the quality of the policy up to an additive error term that telescopes over the iterations. In other words, we show that each iteration improves over the last one in an \emph{average} sense. This is in stark contrast to arguments used for analyzing previous optimistic planning methods that relied on contraction arguments which guarantee \emph{strict} improvements to the policy in each iteration. The advantage is that it remains applicable even when the approximate dynamic programming operator is not contractive or monotone (even approximately).

Our concrete contribution is applying the above scheme to discounted linear mixture MDPs and showing that it achieves a near-optimal regret bound of order $\sqrt{\sprt{B^2 d H + d^2 H^3 + \log \abs{\calA} H^4} T}$, where $d$ is the feature dimension, $B$ is a bound on the norm of the features, and $H = \frac{1}{1-\gamma}$ is the effective horizon. This result implies that our algorithm produces an $\varepsilon$-optimal policy after about $\sprt{B^2 d H + d^2 H^3 + \log \abs{\calA} H^4} / \varepsilon^2$ iterations. Each policy update takes poly($d,H,T$) iterations of regularized dynamic programming, each consisting of poly($d,H,T$) elementary operations. This is to be contrasted with previous contributions on a similar\footnote{We provide a detailed discussion about the differences between our settings in Section~\ref{sec:conc-results}.} setting by \citet*{zhou2021provably}, whose policy updates rely on a version of EVI adapted to linear function approximation. Their EVI variants require globally constraining the model parameters in a way that the model is a valid transition kernel. While this last constraint allowed them to reason about contractive properties of the EVI iterations, it is practically impossible to enforce without making strong assumptions on the feature maps and the MDP itself. The difficulty remains even when the property is only required to hold locally in each state. In this sense, our method is the first to obtain near-optimal statistical rates while also being entirely computationally feasible.

The rest of this paper is organized as follows. After presenting the notation at the end of this section and the background in Section~\ref{sec:prelim}, we introduce our algorithmic framework in Section~\ref{sec:algorithm}. We provide generic performance guarantees and explain the key steps of the analysis in Section~\ref{sec:analysis}. The guarantees are instantiated in the context of tabular and linear mixture MDPs in Section~\ref{sec:applications}. We conclude in Section~\ref{sec:conc} with a discussion of our contribution along with its limitations.

\paragraph{Notation.} For a natural number $N > 0$, we denote $\sbrk{N} = \scbrk{1, 2, \dots, N}$. For a real number $M$, we define the truncation operator $\Pi_M$ that acts on functions $f$ defined on a domain $A$ via $\Pi_M f: x \mapsto \max \sprt{\min \sbrk{f \sprt{x}, M}, 0}$. For a measurable space $(A,\mathcal{F})$, we define the set of all probability distributions $\Delta (A)$, and for any two distributions $P, Q \in \Delta(A)$ such that $P \ll Q$, we define the relative entropy as $\KL \sprt{P \| Q} = \bbE_{a \sim P} \sbrk{\ln \sprt{\frac{\dd P}{\dd Q} (a)}}$. For a distribution $P\in\Delta(A)$ and a bounded function $f\in\real^A$, we write $\iprod{P}{f} = \EEs{f(a)}{a\sim P}$ to denote the expectation of $f$ under $P$, and we will use the same notation for finite-dimensional vector spaces to denote inner products. For a finite-dimensional vector $v\in\real^d$ and a square matrix $Z\in\real^{d\times d}$, we will use the notation $\norm{v}_{Z} = \sqrt{\iprod{v}{Zv}}$.

%% file: 2_preliminaries.tex

\section{Preliminaries}\label{sec:prelim}
We consider a discounted MDP $\calM = \sprt{\calX, \calA, r, P, \gamma, \nu_0}$, where $\calX$ is the finite state space\footnote{Our results extend to the case where $\calX$ is a measurable space. The precise definitions require measure-theoretic concepts \cite{bertsekas1996stochastic}. For the sake of readability and because they are well understood, we only consider finite state spaces here.}, $\calA$ is the finite action space, $r: \calX \times \calA \rightarrow \sbrk{0,1}$ is the deterministic reward function assumed to be known\footnote{It is a standard assumption, and removing it only costs a constant factor in the regret \cite{jaksch2010near}.}, $P: \calX \times \calA \rightarrow \Delta \sprt{\calX}$ is the transition probability distribution, $\gamma \in \sprt{0, 1}$ is the discount factor, and $\nu_0 \in \Delta \sprt{\calX}$ is the initial state distribution. The model describes a sequential interaction scheme between a decision-making \emph{agent} and its environment, where the following steps are repeated for a sequence of rounds $t = 1, 2, \dots$ after the initial state is drawn as $X_0 \sim \nu_0$: the agent observes the state $X_t \in \calX$, selects an action $A_t \in \calA$, obtains reward $r \sprt{X_t, A_t}$, and the environment generates the next state $X_{t+1} \sim P \sprt{\cdot \given X_t, A_t}$. The goal of the agent is to pick its sequence of actions in a way that the total discounted return $\sum_{t=0}^\infty \gamma^{t} r \sprt{X_t, A_t}$ is as large as possible.

Below we describe the most fundamental objects relevant to our work, and refer the reader to the classic book of \citet{puterman2014markov} for more context and details. A (stationary) policy is a mapping $\pi: \calX \rightarrow \Delta \sprt{\calA}$ from a state to a probability measure over actions. The value function and action-value function of a policy $\pi$ are respectively defined as the functions $V^\pi \in \real^\calX$ and $Q^\pi \in \real^{\calX \times \calA}$ mapping each state $x$ and state-action pair $x, a$ to
\begin{align*}
    V^\pi(x) &= \bbE_\pi \sbrk{\sum_{t=0}^\infty \gamma^t r \sprt{X_t, A_t} \given X_0 = x},\\
    Q^\pi (x, a) &= \bbE_\pi \sbrk{\sum_{t=0}^\infty \gamma^t r \sprt{X_t, A_t} \given X_0 = x, A_0 = a},
\end{align*}
where $\bbE_\pi$ denotes the expectation with respect to the probability measure $\bbP_\pi$, generated by the interaction between the environment and the policy $\pi$. With some abuse of notation, we define the conditional expectation operator $P: \real^\calX \ra \real^{\calX\times\calA}$ as $\sprt{P f} \sprt{x, a} = \sum_{x' \in \calX} P \sprt{x' \given x, a} f \sprt{x'}$, for $f \in \real^\calX$. Its adjoint $P\transpose$ acts on distributions $\mu \in \Delta(\calX \times \calA)$ as $\sprt{P\transpose \mu} \sprt{x'} = \sum_{x, a \in \calX \times \calA} P \sprt{x' | x, a} \mu \sprt{x, a}$. It returns the state distribution realized after starting from the state-action distribution $\mu$ and then taking a step forward in the MDP dynamics. With these, we can simply state the \emph{Bellman equations} tying together the value functions as
\begin{align*}
    V^\pi \sprt{x} &= \bbE_{a \sim \pi \sprt{\cdot \given x}} \sbrk{Q^\pi \sprt{x, a}},\\ 
    Q^\pi &= r + \gamma P V^\pi.
\end{align*}
We also introduce the operator $E: \real^\calX \ra \real^{\calX \times \calA}$ acting on functions  $f \in \real^\calX$ via the assignment $\sprt{E f} \sprt{x ,a} = f \sprt{x}$, and its adjoint via its action $E\transpose \mu \sprt{x} = \sum_a \mu \sprt{x, a}$ on distributions $\mu \in \Delta(\calX \times \calA)$. The operator $E$ can be thought of as a ``padding'' operator over the action space that allows us to use vector notation, while $E\transpose$ applied to a state-action distribution returns the corresponding marginal distribution of states. The adjoint $P\transpose$ (resp. $E\transpose$) is the operator such that, for any $f, g$, $\inp{P f, g} = \inp{f, P\transpose g}$ (resp. $E$, $E\transpose$).

In a discounted MDP, a policy $\pi$ induces a unique \emph{normalized discounted occupancy measure} over the state space, defined for any state $x \in \calX$ as
\begin{equation*}
    \nu^\pi \sprt{x} = \sprt{1 - \gamma} \sum_{t=0}^\infty \gamma^t \bbP_\pi \sbrk{X_t = x}.
\end{equation*}
The normalization term $\sprt{1 - \gamma}$ guarantees $\nu^\pi$ is a probability measure over $\calX$. We call the inverse of this normalization constant the \emph{effective horizon} and denote it by $H = \frac{1}{1 - \gamma}$. We also define the associated state-action occupancy measure $\mu^\pi$, defined as $\mu^\pi \sprt{x, a} = \nu^\pi \sprt{x} \pi \sprt{a \given x}$. State-action occupancy measures are known to satisfy the following recurrence relation that is sometimes called the system of \emph{Bellman flow constraints}:
\begin{equation}\label{eq:bellman-flow}
    E\transpose \mu^\pi = \gamma P\transpose \mu^\pi + \sprt{1 - \gamma} \nu_0.
\end{equation}
Using the state-action occupancy measure, the discounted return of a policy can be written as $R_\gamma^\pi = \frac{1}{1 - \gamma} \iprod{\mu^\pi}{r}$. We will use $\mu^*$ to denote an occupancy measure with maximal return and $\nu^* = E\transpose \mu^*$ to denote the associated state-occupancy measure. Finally, given two policies $\pi, \pi'$, we denote $\KL \sprt{\pi \| \pi'} = \sprt{\KL \sprt{\pi \sprt{\cdot \given x} \| \pi' \sprt{\cdot \given x}}}_{x \in \calX}$, and we define $\calH \sprt{\pi \| \pi'} = \inp{\nu^\pi, \KL \sprt{\pi \| \pi'}}$, the  conditional relative entropy\footnote{Technically, this is the conditional relative entropy between the \emph{occupancy measures} $\mu^\pi$ and $\mu^{\pi'}$, but we will keep referring to it in terms of the policies to keep our notations light. We refer to \citet{NJG17} for further discussion.}.\looseness=-1

In this paper, we will consider the setting of online learning in discounted MDPs, where the agent aims to produce an $\varepsilon$-optimal policy $\piout$ satisfying $\iprod{\mu^* - \mu^{\piout}}{r} \le \varepsilon$ based on a single stream of experience in the MDP. We will assume that the learner has access to a \reset action that drops the agent back to a state randomly drawn from the initial-state distribution $\nu_0$, and that the learner follows a stationary policy $\pi_t$ in each round $t$. We will measure the performance in terms of the number of samples necessary to guarantee that the output policy is $\varepsilon$-optimal. As an auxiliary performance measure, we will also consider the \emph{expected regret} (or simply, \emph{regret})\footnote{In the related literature, it is more common to define regret as a random variable and bound it with high probability. Our algorithm is only suitable for bounding the expected regret, and thus we only define this quantity here; we defer further discussion to Section~\ref{sec:conc}.} of the learner defined as\looseness=-1
\begin{equation*}
    \mathfrak{R}_T = \EE{\sum_{t=1}^T \bpa{\iprod{\mu^* - \mu^{\pi_t}}{r}}}.
\end{equation*}
It is easy to see that a regret bound can be converted into sample complexity guarantees. In particular, selecting a time index $I$ uniformly at random from $1,\dots,T$ and returning $\piout = \pi_I$ guarantees that
\begin{equation*}
    \bbE \bigl[ \iprod{\mu^* - \mu^{\piout}}{r} \bigr] = \frac{\mathfrak{R}_T}{T},
\end{equation*}
which can be made arbitrarily small if $\mathfrak{R}_T$ grows sublinearly and $T$ is set large enough. We note here that, while superficially similar to the discounted regret criterion considered in earlier works like \citet{LS20,HZG21} or \citet{zhou2021provably}, there are some major differences between our objectives. We only point out here that we consider the complexity of producing a good policy to execute from the initial state distribution, whereas theirs measures the suboptimality of the policies along the trajectory traversed by the learner. We defer a further discussion of the two settings to Section~\ref{sec:conc-results}.

%% file: 3_algorithm.tex

\section{Algorithm}\label{sec:algorithm}
Our approach implements the principle of optimism in the face of uncertainty in discounted MDPs. Instead of aiming to solve an optimistic version of the Bellman optimality equations via extended value iteration as done by \citet{jaksch2010near}, our method draws on techniques from convex optimization aiming at \emph{average policy improvement}. In particular, our planning procedure is based on a \emph{regularized} version of approximate value iteration and incorporates an \emph{optimistic} estimate of the associated Bellman operator. Consequently, we refer to our algorithm as \MainAlg, standing for Regularized Approximate Value Iteration with Upper Confidence Bounds.

\MainAlg performs a sequence of regularized Q-function and policy updates as follows. Starting with an initial estimate $V_0 = 0$ and an initial policy $\pi_0$, it calculates a sequence of updates for $k = 1, \dots, K$ as
\begin{align*}
    Q_{k+1} (x, a) &= \Pi_H \!\sbrk{r (x, a) + \CB_k (x, a) + \gamma \bpa{\hP V_k} (x, a)}, \\
    V_{k+1} (x) &= \frac 1\eta \log \sprt{\sum_a \pi_k (a | x) e^{\eta Q_{k+1} (x, a)}}, \\
    \pi_{k+1} (a | x) &= \frac{\pi_k (a|x) e^{\eta Q_{k+1} (x, a)}}{\sum_{a'} \pi_k \sprt{a'|x} e^{\eta Q_{k+1} (x, a')}}.
\end{align*}
Here, $\hP$ is a nominal transition model and $\CB_k$ is an exploration bonus defined to be large enough to ensure that $\gamma \hP V_k + \CB_k \ge \gamma P V_k$ and so that $Q_{k+1}$ is an upper bound on the regularized Bellman update $r + \gamma P V_k$. The Q-functions are truncated to the range $\sbrk{0, H}$ to make sure that the optimistic property above can be ensured by setting a reasonably sized exploration bonus $\CB_k$. It is important to note that $Q_{k + 1}$ does not directly attempt to approximate the optimal action-value function $Q^*$ in the true MDP, which marks a clear departure from previously known optimism-based regret analyses. Instead, our analysis will show that $\sprt{1 - \gamma} \inp{\nu_0, V_k}$ acts as an optimistic estimate of the optimal return $\sprt{1 - \gamma} \inp{\nu_0, V^*}$ in an ``average'' sense, and that the total reward of our algorithm can also be bounded in terms of the same quantity.

The overall procedure is presented as Algorithm~\ref{alg:ravi-ucb}. The algorithm proceeds in a sequence of \emph{epochs} $k = 1, 2, \dots$, where a new epoch is started by taking the \reset action with probability $1 - \gamma$ in each round, which results in epochs of average length $H = \frac{1}{1 - \gamma}$. At the beginning of each epoch, we update the model estimate $\wh{P}_k$ and perform one step of online mirror descent to produce the new policy $\pi_k$ and the associated softmax value function $V_k$. We then update the exploration bonuses $\CB_k$ such that they satisfy, for all $x, a$
\begin{equation}\label{eq:valid_CB}
    \babs{\biprod{\gamma P \sprt{\cdot \given x, a} - \gamma \hP_k \sprt{\cdot \given x, a}}{V_k}} \le \CB_k \sprt{x, a}.
\end{equation}
We will refer to exploration bonuses satisfying the above condition as \emph{valid}. As we will see explicitly in Section~\ref{sec:applications}, the model estimate and bonuses are computed using the data gathered so far, $\calD_{T_k}$, where $T_k$ denotes the first time index of epoch $k$. Finally, we apply an optimistic Bellman update to produce a state-action value estimate $Q_{k + 1}$.

We highlight that the assignments in Algorithm~\ref{alg:ravi-ucb} are only made symbolically for all $x, a$, and a practical implementation will not necessarily need to loop over the entire state-action space. Rather, all quantities of interest can be computed on demand while executing the policy in runtime.

\begin{algorithm}[t]
    \caption{\MainAlg.}
    \begin{algorithmic}
    \label{alg:ravi-ucb}
        \STATE {\bfseries Inputs:} Horizon $T$, learning rate $\eta > 0$, initial value $V_0$, initial policy $\pi_0$.
        \STATE {\bfseries Initialize:} $t = 1$, $Q_1 = E V_0$, $\calD_1 = \emptyset$.
        \FOR{$k = 1, \dots$}
        \STATE $T_k = t$.
        \STATE $\wh{P}_{k} = \texttt{TRANSITION-ESTIMATE} \sprt{\calD_{T_k}}$.
        \STATE $V_k \sprt{x} = \frac1\eta \log \sprt{\sum_a \pi_{k - 1} \sprt{a \given x} e^{\eta Q_k \sprt{x, a}}}$.
        \STATE $\pi_k \sprt{a \given x} = \pi_{k - 1} \sprt{a \given x} e^{\eta \sprt{Q_k \sprt{x, a} - V_k \sprt{x}}}$.
        \STATE $\CB_k = \texttt{BONUS} \sprt{\calD_{T_k}}$.
        \STATE $Q_{k + 1} = \Pi_H\! \sbrk{r + \CB_k + \gamma \wh{P}_k V_k}$.
        \REPEAT
        \STATE Play $a_t \sim \pi_k \sprt{\cdot \given x_t}$.
        \STATE Observe $x_{t + 1}$.
        \STATE Update $\calD_{t + 1} = \texttt{ADD} \sprt{\calD_t, \scbrk{\sprt{x_t, a_t, x_{t + 1}}}}$.
        \STATE $t = t + 1$.
        \STATE With probability $1 - \gamma$, \reset to initial distribution: $x_t \sim \nu_0$ and \textbf{break}.
        \UNTIL{$t = T$}
        \ENDFOR
    \end{algorithmic}
\end{algorithm}

Finally, to make some of the arguments in Section~\ref{sec:analysis} more convenient to state, we introduce some notation. We let $\calT_k = \ev{T_k, T_k + 1, \dots, T_{k + 1} - 1}$ denote the time indices belonging to epoch $k$, and $K \sprt{T}$ denote the total number of epochs. For the sake of analysis, we pad out the trajectory of states and actions with the artificial observations $\sprt{X_{T + 1}, A_{T + 1}, \dots, X_{T^+}, A_{T^+}}$, where $T^+$ is the first time that a reset would have occurred had the algorithm been executed beyond time step $T$. These observations are well-defined random variables, and are introduced to make sure that the last epoch does not require special treatment.

%% file: 4_analysis.tex

\section{Main Result \& Analysis}
\label{sec:analysis}

Our main technical result regarding the performance of \MainAlg is the following regret bound.

\begin{theorem}
    \label{thm:main}
    Let $\scbrk{\pi_k}_k$ and $\scbrk{\CB_k}_k$ be the policies and exploration bonuses produced by \MainAlg over $T$ timesteps, where the input is $\eta = \sqrt{2 \log \abs{\calA} / \sprt{H^2 T}}$, $V_0 = 0$ and any policy $\pi_0$. Suppose that the sequence of bonuses $\scbrk{\CB_k}_k$ is valid in the sense of Equation~\eqref{eq:valid_CB}. Then the policies $\scbrk{\pi_k}_k$ satisfy the following bound:
    \begin{align*}
        \mathfrak{R}_T &\le 2 \bbE \sbrk{\sum_{t = 1}^{T^+} \CB_t \sprt{X_t, A_t}}  + \sqrt{2 H^4 \log \abs{\calA} T} + 2 H^2.
    \end{align*}
\end{theorem}

We present the proof of Theorem~\ref{thm:main} below. In particular, we state a sequence of lemmas whose combination will yield the complete proof. We will provide the proofs that we believe to be most insightful in the main text, and relegate the more technical ones to Appendix~\ref{apx:main-proofs}.

The analysis will be split into two main parts: one pertaining to the general properties of our optimistic planning procedure and to the eventual regret bound that can be derived from it, and one concerning the specifics of the setting considered. In particular, we first analyze \MainAlg using a generic exploration bonus that we will suppose to be ``valid'', and then show in Section~\ref{sec:applications} how to derive such valid exploration bonuses in the concrete settings of tabular MDPs and linear mixture MDPs.

\subsection{Optimistic Planning}

We first study the properties of our optimistic planning procedure, without making explicit references to the setting. For this general analysis, we will fix an epoch index $k$, assume that $\hP_k$ is some estimator of the transition kernel $P$ and that the exploration bonus $\CB_k$ is valid in the sense of Equation~\eqref{eq:valid_CB}. We provide the following inequality that will be useful for bounding the suboptimality gaps.

\begin{lemma}
    \label{lem:bound-q-iterates}
    Let $Q_{k + 1}$ be the state-action value estimate produced by \MainAlg in epoch $k$, with any input, and assume the bonuses $\CB_k$ are valid in the sense of Equation~\eqref{eq:valid_CB}. Then,
    \begin{align*}
        r + \gamma P V_k \leq Q_{k + 1} \leq r + 2 \CB_k + \gamma P V_k,
    \end{align*}
    where $V_k$ is the value estimate defined in Algorithm~\ref{alg:ravi-ucb}.
\end{lemma}

\begin{proof}
    We start by proving the lower-bound. For each state-action pair $\sprt{x, a}$, we need to handle two separate cases corresponding to whether or not $Q_{k + 1} \sprt{x, a}$ is truncated from above. In the first case, we have $Q_{k + 1} \sprt{x, a} = H$, which implies
    \begin{equation}\label{eq:truncation_is_important}
        Q_{k + 1} \sprt{x, a} = H = 1 + \gamma H \ge r \sprt{x, a} + \gamma \bpa{P V_k} \sprt{x, a}.
    \end{equation}
    Here, we have crucially used the condition $V_k \le H$ in the inequality, which was made possible by truncating the Q-functions to the range $\sbrk{0, H}$. In the other case where $Q_{k + 1} \sprt{x, a} \le H$, we use the validity of $\CB_k$ to show the following inequality:
    \begin{align*}
        Q_{k + 1} \sprt{x, a} &\ge r \sprt{x, a} + \CB_k \sprt{x, a} + \gamma \bpa{\hP_k V_k} \sprt{x, a} \\
        &\ge r \sprt{x, a} + \gamma \sprt{P V_k} \sprt{x, a},
    \end{align*}
    where the first inequality is valid even when a truncation from below happens.
    
    For the upper-bound, we proceed similarly and consider the two cases corresponding to whether or not $Q_{k + 1} \sprt{x, a}$ is truncated from below in each state-action pair. First considering the case where $Q_{k + 1} \sprt{x, a} = 0$, we observe that
    \begin{align*}
        Q_{k + 1} \sprt{x, a} = 0 \le r \sprt{x, a} + \gamma \pa{P V_k} \sprt{x, a},
    \end{align*}
    from which the claim follows due to non-negativity of $\CB_k$. As for the other case, we have
    \begin{align*}
        Q_{k + 1} \sprt{x, a} &\leq r \sprt{x, a} + \CB_k \sprt{x, a} + \gamma \bpa{\hP_k V_k} \sprt{x, a}
        \\
        &\le r \sprt{x, a} + 2 \CB_k \sprt{x, a} + \gamma \bpa{P V_k} \sprt{x, a},
    \end{align*}
    where the last step follows from the validity condition on $\CB_k$.
\end{proof}

Our key result regarding the quality of the policies produced by \PlanningAlg is the following.

\begin{lemma}
    \label{lem:optimistic-planning}
    Let $K$ be a fixed number of epochs, and let $\pi_k$ and $\CB_k$ be the policy and exploration bonus produced by \MainAlg in epoch $k$, where the input is $V_0 = 0$, any policy $\pi_0$, and any $\eta > 0$. Suppose that $\scbrk{\CB_k}_k$ is a sequence of valid exploration bonuses in the sense of Equation~\eqref{eq:valid_CB}. Then, the sequence $\scbrk{\pi_k}_k$ satisfies the following bound:
        \begin{align*}
            \sum_{k = 1}^K \sprt{\inp{\mu^*, r} - \inp{\mu^{\pi_k}, r}} \leq&\ 2 \sum_{k = 1}^K \inp{\mu^{\pi_k}, \CB_k} + 2 H \\
            &+ \frac1\eta \calH \sprt{\pi^* \| \pi_0} + \frac{\eta H^3 K}{2}.
        \end{align*}
    \end{lemma}
\begin{proof}
    The main idea of the proof is to show that, under the validity condition of the exploration bonuses, $\sprt{1 - \gamma} \inp{\nu_0, V_k}$ acts as an approximate upper bound on the optimal return $\inp{\mu^*, r}$, up to some additional terms resulting from the use of incremental updates. Thanks to the use of regularization, we can show that these additional terms are small on average, and that the gap between the optimistic value and the return of $\pi_k$ can be bounded in terms of $\inp{\mu^{\pi_k}, \CB_k}$. With this in mind, we begin by rewriting the performance gap of the output policy as follows:
    \begin{equation*}
        \sum_{k = 1}^K \sprt{\inp{\mu^*, r} - \inp{\mu^{\pi_k}, r}} = \sum_{k = 1}^K \sprt{\Delta_k^* + \Delta_k},
    \end{equation*}
    where we defined $\Delta_k^* = \inp{\mu^*, r} - \sprt{1 - \gamma} \inp{\nu_0, V_k}$ and $\Delta_k = \sprt{1 - \gamma} \inp{\nu_0, V_k} - \inp{\mu^{\pi_k}, r}$ for all $k$.

    Let us now fix some $k$ and consider the first term, $\Delta_k^*$. We start by observing that $\sprt{1 - \gamma} \nu_0 = E\transpose \mu^* - \gamma P\transpose \mu^*$, which allows us to write
    \begin{equation}\label{eq:delta1}
    \begin{split}
        \Delta_k^* &= \iprod{\mu^*}{r} - \sprt{1 - \gamma} \inp{\nu_0, V_k} \\
        &= \inp{\mu^*, r + \gamma P V_k} - \inp{\mu^*, E V_k}.
    \end{split}
    \end{equation}
    In order to treat the first term in Equation~\eqref{eq:delta1}, we use the lower-bound from Lemma~\ref{lem:bound-q-iterates} to obtain
    \begin{align*}
        \Delta_k^* &\leq \inp{\mu^*, Q_{k + 1} - E V_k} \\
        &= \inp{\mu^*, Q_{k + 1} - E V_{k + 1}} + \inp{\mu^*, E V_{k + 1} - E V_k}.
    \end{align*}
    Summing up for all $k = 1, \dots, K$, we get
    \begin{align*}
        \sum_{k = 1}^K \Delta_k^* &\leq \inp{\mu^*, \overline{Q}_{K + 1} - E\overline{V}_{K + 1}} \\
        &\qquad + \inp{\mu^*, E \pa{V_{K + 1} - V_1}},
    \end{align*}
    where we defined $\overline{Q}_k = \sum_{i = 1}^k Q_i$ and $\overline{V}_k = \sum_{i = 1}^k V_i$ for any $k$. By a classic telescoping argument (presented in Lemma~\ref{lem:md-kl-update}), one can show that, for all $k$,
    \begin{align*}
        &\overline{V}_k \sprt{x} = \max_{p \in \Delta \sprt{\calA}} \scbrk{\iprod{p}{\overline{Q}_k \sprt{x, \cdot}} - \frac1\eta \KL \sprt{p \| \pi_0 \sprt{\cdot \given x}}} \\
        &\qquad \ge \iprod{\pi^* \sprt{\cdot \given x}}{\overline{Q}_k \sprt{x, \cdot}} - \frac1\eta \KL \bpa{\pi^* (\cdot|x)\| \pi_0 (\cdot|x)} .
    \end{align*}
    Combining this with the previous inequality, we get
    \begin{align}
        \sum_{k=1}^K \Delta_k^* &\leq \frac 1\eta \calH \sprt{\pi^* \| \pi_0} + \inp{\mu^*, E V_{K + 1}}, \label{eq:delta1bound}
    \end{align}
    by definition of the conditional entropy and $V_1 = 0$.
    
    We now move on to bounding $\Delta_k$.  Then, using the upper-bound of Lemma~\ref{lem:bound-q-iterates} to lower-bound $r$, we bound $\Delta_k$ as follows:
    \begin{align*}
        \Delta_k &= \sprt{1 - \gamma} \inp{\nu_0, V_k} - \inp{\mu^{\pi_k}, r} \\
        &\leq \sprt{1 - \gamma} \inp{\nu_0, V_k} - \biprod{\mu^{\pi_k}}{Q_{k + 1} - 2 \CB_k - \gamma P V_k} \\
        &= \inp{E\transpose \mu^{\pi_k} - \gamma P\transpose \mu^{\pi_k}, V_k} \\
        &\qquad - \inp{\mu^{\pi_k}, Q_{k + 1} - \gamma P V_k} + 2 \inp{\mu^{\pi_k}, \CB_k},
    \end{align*}
    where we have used $\sprt{1 - \gamma} \nu_0 = E\transpose \mu^{\pi_k} - \gamma P\transpose \mu^{\pi_k}$ in the third line. We can then rewrite the current upper-bound as
    \begin{align*}
        \Delta_k &\leq \inp{\mu^{\pi_k}, E V_k - Q_{k + 1}} + 2 \inp{\mu^{\pi_k}, \CB_k} \\
        &= \inp{\mu^{\pi_k}, E V_k} - \inp{\mu^{\pi_{k + 1}}, Q_{k + 1}} \\
        &\qquad + \inp{\mu^{\pi_{k + 1}} - \mu^{\pi_k}, Q_{k + 1}} + 2 \inp{\mu^{\pi_k}, \CB_k}.
    \end{align*}
    To proceed, we use Lemma~\ref{lem:md-kl-update} to note that
    \begin{equation*}
        \iprod{\mu^{\pi_{k + 1}}}{Q_{k + 1}} = \biprod{E\transpose \mu^{\pi_{k + 1}}}{V_{k + 1} + \frac1\eta \KL \sprt{\pi_{k + 1} \| \pi_k}},
    \end{equation*}
    which allows us to continue as
    \begin{align}
        \Delta_k &\le \inp{\mu^{\pi_k}, E V_k} - \inp{\mu^{\pi_{k + 1}}, E V_{k + 1}} \nonumber \\
        &\quad + \inp{\mu^{\pi_{k + 1}} - \mu^{\pi_k}, Q_{k + 1}} - \frac1\eta \calH \sprt{\pi_{k + 1} \| \pi_k} \label{eq:deltabound_almost_there} \\
        &\quad + 2 \inp{\mu^{\pi_k}, \CB_k} \nonumber.
    \end{align}
    The last remaining difficulty is to control the second difference in the last inequality. This can be done thanks to the regularization, that makes the occupancy measures change ``slowly enough''. To proceed, we use Pinsker's inequality and the boundedness of $Q_{k + 1}$ to show
    \begin{equation*}
        \inp{\mu^{\pi_{k + 1}} - \mu^{\pi_k}, Q_{k + 1}} \leq H \sqrt{2 \KL \sprt{\mu^{\pi_{k + 1}} \| \mu^{\pi_k}}}.
    \end{equation*}
    Appealing to Lemma~\ref{lem:kl-to-centropy}, we can bound the last term as
    \begin{equation*}
        \KL \sprt{\mu^{\pi_{k + 1}} \| \mu^{\pi_k}} \leq H \cdot \calH \sprt{\pi_{k + 1} \| \pi_k}.
    \end{equation*}
    Using these results, we obtain
    \begin{align*}
        &\inp{\mu^{\pi_{k + 1}} - \mu^{\pi_k}, Q_{k + 1}} - \frac1\eta \calH \sprt{\pi_{k + 1} \| \pi_k} \\
        &\qquad \le \sqrt{2 H^3 \calH \sprt{\pi_{k + 1} \| \pi_k}} - \frac1\eta \calH \sprt{\pi_{k + 1} \| \pi_k} \\
        &\qquad \leq \sup_z \scbrk{\sqrt{2 H^3} \cdot z - \frac1\eta z^2} = \frac{\eta H^3}{2},
    \end{align*}
    where the last step follows from the Fenchel--Young inequality applied to the convex function $f \sprt{z} = z^2 / 2$. Then, summing up both sides of Equation~\eqref{eq:deltabound_almost_there} for all $k = 1, \dots, K$,
    \begin{align}
        \sum_{k = 1}^K \Delta_k &\le - \inp{\mu^{\pi_{K + 1}}, E V_{K + 1}} + \frac{\eta H^3}{2} K \nonumber \\
        &\qquad + 2\sum_{k = 1}^K \inp{\mu^{\pi_k}, \CB_k}, \label{eq:delta2bound}
    \end{align}
    where we used $V_1 = 0$. Combining Equations~\eqref{eq:delta1bound} and~\eqref{eq:delta2bound},
    \begin{align*}
        \sum_{k = 1}^K \sprt{\inp{\mu^*, r} - \inp{\mu^{\pi_k}, r}} \leq&\ 2 \sum_{k = 1}^K \inp{\mu^{\pi_k}, \CB_k} + 2 H \\
        &+ \frac1\eta \calH \sprt{\pi^* \| \pi_0} + \frac{\eta H^3}{2} K,
    \end{align*}
    where we used $\inp{\mu^* - \mu^{\pi_{K + 1}}, E V_{K + 1}} \leq 2 H$.
\end{proof}

\subsection{The Epoch Schedule}
The final part is to account for the effects of the randomized epoch schedule. Provided that the exploration bonuses are valid, we need to control the sum $\sum_{t = 1}^T \inp{\mu^{\pi_t}, \CB_t}$. We relate it to a more easily tractable sum in the next lemma.

\begin{lemma}\label{lem:mixing}
The sequence of policies selected by \MainAlg satisfies
\begin{equation*}
    \EE{\sum_{t = 1}^T \iprod{\mu^{\pi_t}}{\CB_t}} \le \EE{\sum_{t = 1}^{T^+} \CB_t \sprt{X_t, A_t}}.
\end{equation*}
\end{lemma}

The proof is in Appendix~\ref{app:mixing}. This bound is guaranteed by the epoch schedule used by \MainAlg that ensures that within each epoch $k$ of geometric length, the sequence of realized state-action trajectory is distributed according to the occupancy measure of $\pi_k$.

\subsection{Putting Everything Together}\label{eq:finishing_up}

The proof of Theorem~\ref{thm:main} concludes by combining the above claims. In anticipation of Section~\ref{sec:applications}, for our main assumption to be satisfied we let $\delta = 1 / T$ and define the exploration bonuses as in Lemma~\ref{lem:CB-tabular} or Lemma~\ref{lem:CB-linear-mixture}. This implies the resulting exploration bonuses are valid with probability at least $1 - \delta$, so on this event we can use Lemma~\ref{lem:optimistic-planning} to bound the expected regret of \MainAlg. Setting $\pi_0$ as the uniform policy, we get
\begin{align*}
    \mathfrak{R}_T &\le 2 \bbE \sbrk{\sum_{t = 1}^T \inp{\mu^{\pi_t}, \CB_t}} \\
    &\quad + H \bbE \sbrk{\frac1\eta \log \abs{\calA} + \frac{\eta H^3}{2} K \sprt{T} + 2 H},
\end{align*}
where we used that the expected epoch length is $H$ and $\calH \sprt{\pi^* \| \pi_0} \leq \log \abs{\calA}$. Noticing that $\bbE \sbrk{K} = \sprt{1 - \gamma} T$ and setting the learning rate $\eta = \sqrt{2 \log \abs{\calA} / \sprt{H^2 T}}$, the expected optimization error becomes
\begin{equation*}
    \bbE \sbrk{\frac1\eta \log \abs{\calA} + \frac{\eta H^3 K}{2}} = \sqrt{2 H^2 T \log \abs{\calA}}.
\end{equation*}
The remaining terms in the regret bound corresponding to the sum of exploration bonuses can be bounded by appealing to Lemma~\ref{lem:mixing}. This concludes the proof.

%% file: 5_applications.tex

\section{Applications}\label{sec:applications}

We now consider two classes of MDPs and show how to implement our algorithm and derive a regret bound.

    \subsection{Tabular MDPs}

For didactic purposes, we first instantiate \MainAlg in the setting of tabular MDPs with small state and action spaces. As we will see, a simple application of our framework allows us to recover known guarantees in this setting. The algorithm can be found in Appendix~\ref{app:ravi-ucb-tabular}. Let $N_1 (x,a) = 1$ and $N_1' (x,a,x') = 0$ denote the initial counts\footnote{We initialize $N_1$ at 1 to avoid divisions by zero.} for the tuples $\sprt{x, a}$ and $(x,a,x')$. At epoch $k$, for $t \in \calT_k$, we update $\calD_{t + 1} = \sprt{N_{t + 1}, N_{t + 1}'}$ as $N_{t + 1} (x,a) = N_t (x,a) + \II{X_t = x, A_t = a}$ and $N_{t + 1}' (x,a,x') = N_t' (x,a,x') + \II{X_t = x, A_t = a, X_{t + 1} = x'}$. We use $\wh{P}_k \sprt{x' \given x, a} = N_{T_k} (x, a, x') / N_{T_k} (x, a)$ as a model estimate, and given $\beta > 0$, the exploration bonuses are defined as
\begin{equation}\label{eq:cb-tabular}
    \CB_k (x,a) = \frac{\beta}{\sqrt{N_{T_k} (x,a)}}.
\end{equation}
The following lemma shows that an appropriate choice of the scaling parameter $\beta$ ensures the validity of the exploration bonuses.
\begin{lemma}\label{lem:CB-tabular}
    Let $\delta \in \sprt{0, 1}$. Then, setting the coefficient $\beta = 8 H \sqrt{\abs{\calX} \log \sprt{\abs{\calX} \abs{\calA} T / \delta}}$ guarantees that, with probability $1 - \delta$, the validity condition \eqref{eq:valid_CB} is satisfied by $\CB_k$ as defined in Equation~\eqref{eq:cb-tabular} for all $k$.
\end{lemma}
Then, we can bound the bonuses as follows.
\begin{lemma}\label{lem:elliptical-tabular}
    The sum of exploration bonuses generated by \MainAlg satisfies
    \begin{equation*}
        \bbE \sbrk{\sum_{t = 1}^{T^+} \CB_t \sprt{X_t, A_t}} = \calO \sprt{\beta \sqrt{\abs{\calX} \abs{\calA} T}}.
    \end{equation*}
\end{lemma}
We refer the reader to previous works for the proofs of the above two lemmas (see, \eg, \citealp{jaksch2010near,fruit2018efficient}). Combining the above two results gives a regret bound of order ${\abs{\calX} H \sqrt{\abs{\calA} T}}$, as expected.

    \subsection{Linear Mixture MDPs}

We now focus on a class of MDPs commonly referred to as \emph{linear mixture MDPs} \cite{MJTS20, AJSzWY20} formally defined as follows.
\begin{assumption}[Linear mixture MDP]
\label{asp:linear-mixture-mdp}
    There exist a known feature map $\psi: \calX \times \calA \times \calX \ra \bbR^d$, and an unknown $\theta \in \bbR^d$ with $\norm{\theta}_2 \leq B$ such that $P (x' | x, a) = \sum_{i = 1}^d \theta_i \psi_i (x, a, x')$. Furthermore, for any $(x,a) \in \calX \times \calA$, $V \in \sbrk{0, H}^\calX$,
    \begin{equation*}
        \norm{\sum_{x' \in \calX} \psi (x,a,x') V \sprt{x'}}_2 \leq B H.
    \end{equation*}
\end{assumption}
Here, we suppose $\calM$ satisfies Assumption~\ref{asp:linear-mixture-mdp}. While remotely related to the notion of linear MDPs \cite{jin2020provably, YW19}, linear mixture MDPs are a distinct class of models that cannot be captured in that framework, and have been widely studied in the past few years as linear MDPs---we refer to \citet{zhou2021provably} for further discussion. As often assumed in the related literature, we assume the map $\phi_k (x,a) = \sum_{x'} \psi (x,a,x') V_k \sprt{x'}$ can be computed (or approximated) efficiently. We provide a detailed discussion of all such computational matters in Section~\ref{sec:conc-limitations}.

The algorithm is in Appendix~\ref{app:ravi-ucb-mixture}. Let $\lambda > 0$ be a regularization parameter, $\Lambda_1 = \lambda I$, and $b_1 = 0$. At epoch $k$, for $t \in \calT_k$, the data is stored as $\calD_{t + 1} = \sprt{\Lambda_{t + 1}, b_{t + 1}}$ where $\Lambda_{t + 1} = \Lambda_t + \phi_k \sprt{x_t, a_t} \phi_k \sprt{x_t, a_t}\transpose$ and $b_{t + 1} = b_t + \phi_k \sprt{x_t, a_t} V_k \sprt{x_{t + 1}}$. $\wh{P}_k = \sum_i \wh{\theta}_{k, i} \psi_i$ is computed via a least-squares regression, where $\wh{\theta}_k = \Lambda_{T_k}^{-1} b_{T_k}$. Given $\beta > 0$, the exploration bonuses are defined as
\begin{equation}\label{eq:cb-mixture}
    \CB_k (x,a) = \beta \norm{\phi_k (x,a)}_{\Lambda_{T_k}^{-1}}.
\end{equation}
We now turn to the validity condition required by Lemma~\ref{lem:optimistic-planning}.

\begin{lemma}\label{lem:CB-linear-mixture}
    Let $\delta \in \sprt{0, 1}$. Then, setting the coefficient $\beta = H \sqrt{2 \sprt{\frac{d}{2} \log \sbrk{1 + \frac{T B^2 H^2}{\lambda d}} + \log \frac{1}{\delta}}} + \sqrt{\lambda} B$ guarantees that, with probability $1 - \delta$, the validity condition \eqref{eq:valid_CB} is satisfied by $\CB_k$ as defined in Equation~\eqref{eq:cb-mixture} for all $k$.
\end{lemma}

The proof is in Appendix~\ref{app:CB}. It relies on standard techniques regarding linear mixture MDPs \cite{zhou2021provably, cai2020provably}. One important property required is the boundedness of each $V_k$ that is guaranteed by the truncation. Then, we can bound the sum of the exploration bonuses with the following lemma.

\begin{lemma}\label{lem:elliptical}
    The sum of exploration bonuses generated by \MainAlg satisfies
    \begin{align*}
        \bbE \sbrk{\sum_{t = 1}^{T^+} \CB_t \sprt{X_t, A_t}} = \calO \sprt{\beta \sqrt{d H T} \log \sprt{T}}.
    \end{align*}
\end{lemma}
The proof (Appendix~\ref{app:elliptical}) follows from a series of small (but somewhat tedious) adjustments of a classic result often referred to as the ``elliptical potential lemma'', the main challenge being dealing with the randomized epoch schedule.

Our main technical result regarding the performance of \MainAlg is the following.
\begin{theorem}
    \label{thm:regret-mixture-mdp}
    Suppose that \MainAlg is run with the uniform policy as $\pi_0$, $V_0 = 0$, $\lambda = 1$, a learning rate $\eta = \sqrt{2 \log \abs{\calA} / \sprt{H^2 T}}$, and an exploration parameter $\beta = H \sqrt{2 \sprt{\frac{d}{2} \log \sbrk{1 + \frac{T B^2 H^2}{d}} + \log T}} + B$. Then, the expected regret of \MainAlg satisfies
    \begin{equation*}
        \mathfrak{R}_T = \wt{\calO} \sprt{ \sqrt{\sprt{d^2 H^3 + B^2 d H +  H^4 \log \abs{\calA}} T}}.
    \end{equation*}
\end{theorem}
$\widetilde{\calO} \sprt{\cdot}$ hides logarithmic factors of $T$, $B$, $d$, and $H$. A perhaps more useful result is the following, derived from an online-to-batch conversion. Suppose \MainAlg returns a policy $\pi_{\text{out}} = \pi_U$  with $U$ being an epoch index chosen uniformly at random from the range of epochs. The following corollary provides a guarantee on the quality of this policy.
\begin{corollary}\label{cor:online-to-batch}
    Let $\varepsilon > 0$. Then, \MainAlg run with the same parameters as before outputs a policy $\pi_{\text{out}}$ satisfying $\bbE \sbrk{\iprod{\mu^* - \mu^{\pi_{\text{out}}}}{r}} \le \varepsilon$ after $T_\epsilon$ steps, with
    \begin{equation*}
        T_\varepsilon = \wt{\calO} \sprt{\sprt{B^2 d H + d^2 H^3 + H^4\log \abs{\calA}} \varepsilon^{-2}}.
    \end{equation*}
\end{corollary}
The expectation appearing in the first statement is with respect to the random transitions in the MDP and the epoch scheduling, whereas the expectation in the second one is also with respect to the random choice of the policy. It is possible to remove the former expectation, but the latter is inherent to the online-to-batch conversion process used by our analysis. We will return to this point in Section~\ref{sec:conc-limitations}.

%% file: 6_conclusion.tex

\section{Discussion}\label{sec:conc}
We now discuss the merits and limitations of our results, and point out directions for future research.

    \subsection{Results and Comparisons}\label{sec:conc-results}

There are many differences between our approach and previously proposed optimistic exploration methods that we are aware of. Perhaps the most interesting novelty in our method is that it radically relaxes the optimistic properties that previous methods strive for: instead of calculating estimates of the value function or the MDP model that are strictly optimistic, we only guarantee that our value estimates are optimistic in an average sense. Thus, during its runtime, our algorithm may execute several policies that do not individually satisfy any optimistic properties, even approximately. We find this property to be curious and believe that the ideas we develop to tackle such notions of ``average optimism'' may find other applications. We note though that our planning procedure can be used to produce optimistic policies in a stricter sense by executing several regularized value iteration steps per policy update, until the resulting optimization error vanishes. Doing so results in an improved dependency on $H$ by a factor $\sqrt{H}$ but comes at the cost of a major computational overhead.

While our algorithm is closely related to the MD-MPI method of \citet*{geist2019theory} and our proofs feature several similar steps, we remark that the purpose of our analysis is quite different from theirs, even when disregarding the optimistic adjustment we make to the Bellman operators. Taking a close look at their proofs for the special case of zero approximation errors, one can deduce bounds on our quantity of interest that are of the order $\pa{H + \mathcal{H(\pi^*\|\pi_K)}}/K$ after $K$ iterations. This is faster than what our analysis provides for approximate DP, which is due to the monotonicity of the exact Bellman operator which allows fast last-iterate convergence. The same rate appears in the analysis of regularized policy iteration methods by \citet{AKLM21} (see Theorem~16). Either way, all of these analyses use tools from the analysis of mirror descent first developed by \citet{Martinet1970}, \citet{R76}, and \citet{NY83} (see also \citealp{BT03}). Note that, as the guarantees of these regularization-based methods hold on arbitrary data sequences, our regret guarantees trivially generalize to the case where the rewards change over time in a potentially adversarial fashion (as in, e.g., \citealp{EDKSM09,cai2020provably}). 

\looseness=-1

Another line of work that our contribution seemingly fits into is the one initiated by \citet{LS20} on the topic of regret minimization for discounted MDPs (see also \citealp{HZG21,zhou2021provably}). A closer look reveals that their objective is quite different from ours, in that they aim to upper bound $\sum_{t=1}^T \pa{V^* \sprt{X_t} - V^{\pi_t} \sprt{X_t}}$ \emph{along the trajectory traversed by the learning agent}. This notion of regret has been motivated by a formerly popular notion of ``sample complexity of exploration'' in discounted MDPs---we highlight \citet{KakadeThesis:2003,SLL09} out of the abundant ``PAC-MDP'' literature on this subject. This performance measure is in fact not comparable to ours in almost any possible sense. In fact, it is easy to see that this notion may fail to capture the sample complexity of learning a good policy in a meaningful way: a policy that immediately enters a ``trap'' state that yields zero reward until the end of time will only incur a \emph{constant} regret of order $\frac{1}{1 - \gamma}$, even if there is a policy that yields a steady stream of $+1$ rewards in each round. Thus, without making stringent assumptions about the MDP that rule out such undesirable scenarios, the value of minimizing this notion of discounted regret may be questionable.\looseness=-1

    \subsection{Limitations and Future Directions}\label{sec:conc-limitations}

On a related note, our method suffers from the limitation of requiring access to a reset action taking the agent back to the initial distribution $\nu_0$ at any time. In general, this is necessary to achieve our objectives. Indeed, in MDPs where all states around the initial distribution are transient, it is impossible to learn a good policy from a single stream of experience without resets since the agent only gets to visit the relevant part of the state space once. We thus believe these issues are inherent to learning in discounted MDPs.

Another limitation is that our guarantees only hold on expectation as opposed to high probability. In fact, several of our results can be strengthened to hold in this stronger sense, albeit at the cost of a more involved analysis. In particular, the only parts of our analysis that need to be changed are Lemmas~\ref{lem:mixing} and~\ref{lem:elliptical}, to deal with the randomized epoch schedule. The first of these can be handled via a martingale argument and the second by bounding the number and length of the epochs with high probability. Both of these changes are conceptually simple, but practically tedious so we omit them for clarity. On the other hand, Corollary~\ref{cor:online-to-batch} relies on a randomized online-to-batch conversion, and the result is stated on expectation with respect to the randomization step. Once again, this result can be strengthened to hold with high probability by running a ``best-policy-selection'' subroutine on the sequence of policies produced by the algorithm. This post-processing step is standard in the related literature and we omit details here to preserve clarity.

Based on our current results, generalizing our techniques to the infinite-horizon average-reward setting seems to be challenging but not impossible. The key step in our proof that requires discounting is setting the truncation level at $H = \frac{1}{1 - \gamma}$, which serves the purpose of guaranteeing that our approximate Bellman operator is optimistic. In particular, the truncation level needs to be set large enough so that the inequality of Equation~\ref{eq:truncation_is_important} goes through. We see no natural way to extend this condition to the undiscounted setup. We remain hopeful that this challenge can be overcome with more effort (but may potentially need some significant new ideas).\looseness=-1

Finally, let us remark on the linear mixture MDP assumption that we have used. While arguably well-studied in the past years, this model for linear function approximation has limitations that make it rather difficult to adapt to practical scenarios. The biggest is that learning algorithms in this model need access to an oracle to evaluate sums of the form $\sum_{x'} \psi \sprt{x, a, x'} V \sprt{x'}$, which can only be performed efficiently in special cases. Options include assuming that $\psi \sprt{x, a, \cdot}$ is sparse or the integral can be approximated effectively via Monte Carlo sampling. A major inconvenience that this causes in the implementation of our method is that Q-functions (and policies) cannot be represented effectively with a single low-dimensional object, so these values have to be recalculated on the fly while executing the policy, requiring excessive Monte Carlo integration in runtime. We thus wish to extend our analysis to more tractable MDP models like the model of \citet{jin2020provably}. While it is straightforward to implement our algorithm for linear MDPs, unfortunately the covering number of the value function class used by our algorithm appears to be too large to allow proving strong performance bounds. On a more positive note, we wish to point out that linear mixture MDPs are still a rich family of models that in general is incomparable to linear MDPs, and can subsume many interesting models---we refer to \citep{AJSzWY20} for further discussion. We are optimistic that the limitations of our current analysis can be eventually removed and our method can be adapted to a much broader class of infinite-horizon MDPs.\looseness=-1

%% file: 7_acknowledgements.tex

\section*{Acknowledgements}

G.~Neu was supported by the European Research Council (ERC) under the European Union’s Horizon 2020 research and innovation programme (Grant agreement No.~950180). Part of this work was done while the second author was visiting the Simons Institute for the Theory of Computing.

%% file: A_main_proofs.tex

\section{Omitted Proofs}
\label{apx:main-proofs}

\subsection{Technical Tools for the Proof of Lemma~\ref{lem:optimistic-planning}}

\begin{lemma}
\label{lem:kl-to-centropy}
    Let $\pi$ and $\pi'$ be two policies, with their corresponding state-action occupancy measures being $\mu^{\pi}$ and  $\mu^{\pi'}$, and their state occupancy measures being $\nu^{\pi}$ and $\nu^{\pi'}$. Then,
    \begin{equation*}
        \KL \sprt{\mu^\pi \middle\| \mu^{\pi'}} \leq \frac{1}{1 - \gamma} \calH \sprt{\pi \| \pi'}.
    \end{equation*}
\end{lemma}

\begin{proof}
    Using the chain rule of the relative entropy, we write
    \begin{equation*}
        \KL \sprt{\mu^\pi \middle\| \mu^{\pi'}} = \KL \sprt{\nu^\pi \middle\| \nu^{\pi'}} + \calH \sprt{\pi \| \pi'}.
    \end{equation*}
    By the Bellman flow constraints in \cref{eq:bellman-flow} and the joint convexity of the relative entropy, we bound the second term as
    \begin{align*}
        \KL \sprt{\nu^\pi \middle \| \nu^{\pi'}} &= \KL \sprt{\gamma P\transpose \mu^\pi + \sprt{1 - \gamma} \nu_0 \middle\| \gamma P\transpose \mu^{\pi'} + \sprt{1 - \gamma} \nu_0} \\
        &\leq \sprt{1 - \gamma} \KL \sprt{\nu_0 \middle\| \nu_0} + \gamma \KL \sprt{P\transpose \mu^\pi \middle\| P\transpose \mu^{\pi'}} \\
        &= \gamma \KL \sprt{P\transpose \mu^\pi \middle\| P\transpose \mu^{\pi'}} \le \gamma \KL \sprt{\mu^\pi \middle\| \mu^{\pi'}},
    \end{align*}
    where we also used the data-processing inequality in the last step. The proof is concluded by reordering the terms.
\end{proof}

\subsection{Proof of Lemma~\ref{lem:CB-linear-mixture}}\label{app:CB}
Let us fix $k \in \sbrk{K}$, $t \in \scbrk{T_k,T_k + 1,\dots, T_{k + 1} - 1}$, $\delta \in \sprt{0, 1}$. We start by recalling the definition of the nominal transition model $\hP_k$ acting on functions $V$ as $\bpa{\hP_k V} \sprt{x, a} = \biprod{\phi_V \sprt{x, a}}{\wh{\theta}_k}$, where we denoted the state-action feature map $\phi_V \sprt{x, a} = \sum_{x' \in \calX} \psi \sprt{x, a, x'} V \sprt{x'}$, and the parameter $\wh{\theta}_k$ can be written out as
\begin{equation*}
    \wh{\theta}_k = \Lambda_{T_k}^{-1} b_{T_k} = \sprt{\sum_{i = 0}^{k - 1} \sum_{j = T_i}^{T_{i+1} - 1} \phi_i \sprt{x_j, a_j} \phi_i \sprt{x_j, a_j}\transpose + \lambda I}^{-1} \sum_{i = 0}^{k - 1} \sum_{j = T_i}^{T_{i+1} - 1} \phi_i \sprt{x_j, a_j} V_i \sprt{x_{j + 1}}.
\end{equation*}

To proceed, we notice that the true transition operator acting on $V$ can be written in a similar form as
\begin{align*}
    \sprt{P V} \sprt{x, a} &= \sum_{x' \in \calX} P \sprt{x' \given x, a} V \sprt{x'} & \mbox{(by definition of $P$)} \\
    &= \sum_{x' \in \calX} \inp{\theta, \psi \sprt{x, a, x'}} V \sprt{x'} & \mbox{(by Assumption~\ref{asp:linear-mixture-mdp})} \\
    &= \inp{\theta, \sum_{x' \in \calX} \psi \sprt{x, a, x'} V \sprt{x'}} \\
    &= \inp{\theta, \phi_V \sprt{x, a}},
\end{align*}
where we used the definition of $\phi_V$ in the last line. Proceeding further with the same expression, we write
\begin{align*}
    \sprt{P V} \sprt{x, a} &= \inp{\phi_V \sprt{x, a}, \Lambda_{T_k}^{-1} \Lambda_{T_k} \theta} \\
    &= \inp{\phi_V \sprt{x, a}, \Lambda_{T_k}^{-1} \sum_{i = 0}^{k - 1} \sum_{j = T_i}^{T_{i+1} - 1} \phi_i \sprt{x_j, a_j} \phi_i \sprt{x_j, a_j}\transpose \theta + \lambda \Lambda_{T_k}^{-1} \theta} \\
    &= \inp{\phi_V \sprt{x, a}, \Lambda_{T_k}^{-1} \sum_{i = 0}^{k - 1} \sum_{j = T_i}^{T_{i + 1} - 1} \phi_i \sprt{x_j, a_j} \sprt{P V_i} \sprt{x_j, a_j} + \lambda \Lambda_{T_k}^{-1} \theta},
\end{align*}
where we used the definition of $\Lambda_{T_k}$ and Assumption~\ref{asp:linear-mixture-mdp} in the last line. Comparing the expressions for $P V$ and $\hP_k V$, we obtain
\begin{equation*}
    \abs{\hP_k V \sprt{x, a} - P V \sprt{x, a}} = \abs{\inp{\phi_V \sprt{x, a}, \Lambda_{T_k}^{-1} \sum_{i = 0}^{k - 1} \sum_{j = T_i}^{T_{i + 1} - 1} \phi_i \sprt{x_j, a_j} \sbrk{V_i \sprt{x_{j + 1}} - \sprt{P V_i} \sprt{x_j, a_j}} - \lambda \Lambda_{T_k}^{-1} \theta}}.
\end{equation*}
Using the Cauchy--Schwartz inequality and taking $V = V_k$, we get
\begin{equation*}
    \abs{\hP_k V_k \sprt{x, a} - P V_k \sprt{x, a}} \leq \norm{\phi_k \sprt{x, a}}_{\Lambda_{T_k}^{-1}} \sprt{\abs{\xi_k} + \abs{b_k}},
\end{equation*}
where $\xi_k = \norm{\sum_{i = 0}^{k - 1} \sum_{j = T_i}^{T_{i + 1} - 1} \phi_i \sprt{x_j, a_j} \sbrk{V_i \sprt{x_{j + 1}} - \sprt{P V_i} \sprt{x_j, a_j}}}_{\Lambda_{T_k}^{-1}}$, and $b_k = \lambda \norm{
\theta}_{\Lambda_{T_k}^{-1}}$. The second term can be easily bounded as $\abs{b_k} \leq \sqrt{\lambda} \norm{\theta}_2 \leq \sqrt{\lambda} B$, using that $\lambda_{\text{min}} \sprt{\Lambda_{T_k}} \geq \lambda $ and the boundedness of the features.

For the first term, observe that $V_i \sprt{x_{j + 1}} - \sprt{P V_i} \sprt{x_j, a_j}$ forms a martingale difference sequence, with increments bounded in $\sbrk{-H, H}$ by the truncation made in the algorithm. Additionally, the feature vectors are bounded as $\norm{\phi_i \sprt{x_j, a_j}}_2 \leq B H$ and the true parameter as $\norm{\theta}_2 \leq B$ by Assumption~\ref{asp:linear-mixture-mdp}. Therefore, we can apply the self-normalized concentration result in
Theorem~\ref{thm:self-concentration} (stated in Appendix~\ref{app:concentration}), which guarantees that with probability at least $1 - \delta$, the following bound holds for all $k \in \sbrk{K}$:
\begin{equation*}
    \xi_k \leq H \sqrt{2 \log \sbrk{\frac{\det \sprt{\Lambda_{T_k}}^{1 / 2} \det \sprt{\lambda I}^{- 1 / 2}}{\delta}}}.
\end{equation*}
The determinants appearing in the bound can be further upper bounded by using that $\det \sprt{\lambda I} = \lambda^d$ and
\begin{align*}
    \det \sprt{\Lambda_{T_k}} &\leq \sprt{\frac{\trace{\Lambda_{T_k}}}{d}}^d & \mbox{(by the trace-determinant inequality)} \\
    &= \frac{1}{d^d} \sprt{\lambda d + \sum_{i = 0}^{k - 1} \sum_{j = T_i}^{T_{i+1} - 1} \norm{\phi_i \sprt{x_j, a_j}}_2^2}^d & \text{(by the definition of $\Lambda_{T_k}$)} \\
    &\leq \sprt{\lambda + \frac{T_k B^2 H^2}{d}}^d & \text{(by the boundedness of the features)} \\
    &\leq \sprt{\lambda + \frac{T B^2 H^2}{d}}^d,
\end{align*}
where in the last step we used $T_k \le T$. We plug this back in the upper-bound on $\xi_k$ to obtain the bound
\begin{equation*}
    \xi_k \leq H \sqrt{2 \sprt{\frac{d}{2} \log \sbrk{1 + \frac{T B^2 H^2}{\lambda d}} + \log \frac1\delta}}.
\end{equation*}
Putting everything together, we have verified that, for all $k \in \sbrk{K}$,
\begin{align*}
    \abs{\inp{P \sprt{\cdot \given x, a} - \hP_k \sprt{\cdot \given x, a}, V_k}} &\leq \norm{\phi_k \sprt{x, a}}_{\Lambda_{T_k}^{-1}} \sprt{H \sqrt{2 \sprt{\frac{d}{2} \log \sbrk{1 + \frac{T B^2 H^2}{\lambda d}} + \log \frac1\delta}} + \sqrt{\lambda} B} \\
    &= \beta \norm{\phi_k \sprt{x, a}}_{\Lambda_{T_k}^{-1}}.
\end{align*}
holds with probability at least $1 - \delta$, where we have defined $\beta$ as
\begin{equation}
\label{eq:bonus-coef}
    \beta = H \sqrt{2 \sprt{\frac{d}{2} \log \sbrk{1 + \frac{T B^2 H^2}{\lambda d}} + \log \frac1\delta}} + \sqrt{\lambda} B.
\end{equation}
This concludes the proof. \qed

\subsection{Proof of Lemma~\ref{lem:mixing}}\label{app:mixing}
For the sake of this proof, we slightly update our notation for $\calT_k$ by setting $\calT_{K \sprt{T}} = \ev{T_{K \sprt{T}}, T_{K \sprt{T}} + 1, \dots, T^+}$. We will use $\calF_{k - 1}$ to denote the filtration generated by the observations up to the end of epoch $k - 1$, and $L_k$ to denote the length of epoch $k$. We start by rewriting the sum of exploration bonuses up to step $T^+$ as
\begin{equation}\label{eq:CB_martingale}
    \sum_{t = 1}^{T^+} \CB_t \sprt{X_t, A_t} = \sum_{k = 1}^{K \sprt{T}} \sum_{t \in \calT_k} \CB_t \sprt{X_t, A_t}.
\end{equation}
By virtue of the definition of $T^+$, all epochs are of geometric length with mean $\frac{1}{1 - \gamma}$.
Now, let us consider a fixed epoch $k$ and define the auxiliary infinite sequence of state-action pairs $X_{k, 0}, A_{k, 0}, X_{k, 1}, A_{k, 1}, \dots$ that is generated independently from the realized sample trajectory $\sprt{X_t, A_t}_{t \in \calT_k}$ given $\F_{k - 1}$ as follows. The initial state $X_{k, 0}$ is drawn from $\nu_0$, and then subsequently for each $i = 0, 1, \dots$, the actions are drawn as $A_{k, i} \sim \pi_k \sprt{\cdot \given X_{k, i}}$ and follow-up states are drawn as $X_{k, i + 1} \sim P \sprt{\cdot \given X_{k, i}, A_{k, i}}$. Recalling the notational convention that $\CB_t = \CB_k$ for all $t \in \calT_k$, we observe that for any
$k$, we have
\begin{align*}
    \EEcc{\sum_{t \in \calT_k} \CB_t \sprt{X_t, A_t}}{\F_{k - 1}} &= \EEcc{\sum_{i = 0}^{L_k - 1} \CB_k \sprt{X_{k, i}, A_{k, i}}}{\F_{k - 1}} \\
    &= \EEcc{\sum_{i = 0}^\infty \II{i < L_k} \CB_k \sprt{X_{k, i}, A_{k, i}}}{\F_{k - 1}} \\
    &= \EEcc{\sum_{i = 0}^\infty \gamma^i \CB_k \sprt{X_{k, i}, A_{k, i}}}{\F_{k - 1}} \\
    &= \sum_{i = 0}^\infty \gamma^i \iprod{u_{k, i}}{\CB_k} = \frac{\iprod{\mu^{\pi_k}}{\CB_k}}{1 - \gamma} \\
    &= \EEcc{L_k \iprod{\mu^{\pi_k}}{\CB_k}}{\F_{k - 1}} = \EEcc{\sum_{t \in \calT_k} \iprod{\mu^{\pi_k}}{\CB_k}}{\F_{k - 1}},
\end{align*}
where in the third line we have observed that $L_k$ follows a geometric law with parameter $1 - \gamma$, and is independent of $\sprt{X_{k, i}, A_{k, i}}_i$. In the fourth line we introduced the notation $u_{k, i}$ to denote the joint distribution of $X_{k, i}, A_{k, i}$ given $\F_{k - 1}$ and noticed that the discounted sum of these distributions exactly matches the definition of the occupancy measure $\mu^{\pi_k}$ up to the normalization constant $\sprt{1 - \gamma}$, and finally concluded by observing that $\EEcc{L_k}{\F_{k - 1}} = \frac{1}{1 - \gamma}$. 

The proof is completed by summing up over all epochs, taking marginal expectations, and noticing that
\begin{equation*}
    \EE{\sum_{t = 1}^T \iprod{\mu^{\pi_t}}{\CB_t}} \le \EE{\sum_{t = 1}^{T^+} \iprod{\mu^{\pi_t}}{\CB_t}} = \EE{\sum_{k = 1}^{K \sprt{T}} \sum_{t \in \calT_k} \iprod{\mu^{\pi_k}}{\CB_k}}.
\end{equation*}
\qed

\subsection{Proof of \cref{lem:elliptical}}\label{app:elliptical}
The proof is based on a classic ``pigeonhole'' argument often called the ``elliptical potential lemma'' (\eg, Lemma~19.4 in \citealp{LSz20}, or Section~11.7 in \citealp{cesa2006prediction}, but see also \citealp{LRW79,LW82}). The main challenge of adapting this result to our setting is accounting for the randomized epoch schedule. Another subtle difficulty comes from the fact that Lemma~\ref{lem:optimistic-planning} only bounds the total regret as opposed to the instantaneous regrets in each round, which necessitates arguments that are slightly more involved than what is commonly seen in closely related work.

As for the actual proof, we start by introducing some useful notation that we will use throughout the proof. For $t \in \sbrk{T}$, we use $k_t$ to denote the index of the epoch that $t$ belongs to. For simplicity, for all $k$ and $t$, we will write $\phi_{k, t} = \phi_k \sprt{X_t, A_t}$, $\Lambda_k = \Lambda_{T_k}$. We also define $\calN \sprt{T} = \scbrk{t \in \sbrk{T}: \norm{\phi_{k_t, t}}_{\Lambda_{k_t}^{-1}} \geq 1}$ as the set of ``bad'' time indices where state-action pairs with large feature norms are observed, and $\calE \sprt{T} = \scbrk{k \in \sbrk{K \sprt{T}}: \exists t \in \calT_k, \norm{\phi_{k, t}}_{\Lambda_k^{-1}} \geq 1}$ be the set of epochs containing at least one bad time index. Using these definitions, we rewrite the sum of exploration bonuses as follows:
\begin{align*}
    \bbE \sbrk{\sum_{t = 1}^{T^+} \CB_t \sprt{X_t, A_t}} &= \beta \bbE \sbrk{\sum_{k \in \calE \sprt{T}} \sum_{t \in \calT_k} \norm{\phi_{k, t}}_{\Lambda_k^{-1}} + \sum_{k \notin \calE \sprt{T}} \sum_{t \in \calT_k} \norm{\phi_{k, t}}_{\Lambda_k^{-1}}} \\
    &\leq \beta \bbE \sbrk{\frac{B H}{\sqrt{\lambda}} \sum_{k \in \calE \sprt{T}} \abs{\calT_k} + \sum_{k \notin \calE \sprt{T}} \sum_{t \in \calT_k} \norm{\phi_{k, t}}_{\Lambda_k^{-1}}}  & \text{(using $\twonorm{\phi}\le BH$ and $\Lambda_k \succeq \lambda I$)} \\
    &= \beta \bbE \sbrk{\abs{\calE \sprt{T}}} \frac{B H^2}{\sqrt{\lambda}} + \beta \bbE \sbrk{\sum_{k \notin \calE \sprt{T}} \sum_{t \in \calT_k} \norm{\phi_{k, t}}_{\Lambda_k^{-1}}} & \text{(using Wald's identity)} \\
    &= \beta \bbE \sbrk{\abs{\calE \sprt{T}}} \frac{B H^2}{\sqrt{\lambda}} + \beta \bbE \sbrk{\sum_{k \notin \calE \sprt{T}} \sum_{t \in \calT_k} \sprt{1 \wedge \norm{\phi_{k, t}}_{\Lambda_k^{-1}}}},
\end{align*}
where in the last step we used the definition of $\calE(T)$. We treat the first term separately in Lemma~\ref{lem:bad-epochs}, stated after this proof. This gives the following bound:
\begin{equation*}
    \bbE \sbrk{\sum_{t = 1}^{T^+} \CB_t \sprt{X_t, A_t}} \leq \frac{\beta d B H^2}{\sqrt{\lambda} \log \sprt{2}} \log \sprt{1 + \frac{B^2 H^2 T}{\lambda d}} + \beta \bbE \sbrk{\sum_{k \notin \calE \sprt{T}} \sum_{t \in \calT_k} \sprt{1 \wedge \norm{\phi_{k, t}}_{\Lambda_k^{-1}}}}.
\end{equation*}
Thus, we can focus on the second term in the right hand side. This term can be upper-bounded using the Cauchy--Schwarz inequality as
\begin{align*}
    \sum_{k \notin \calE \sprt{T}} \sum_{t \in \calT_k} \sprt{1 \wedge \norm{\phi_{k, t}}_{\Lambda_k^{-1}}} &\leq \sum_{k = 1}^{K \sprt{T}} \sum_{t \in \calT_k} \sprt{1 \wedge \norm{\phi_{k, t}}_{\Lambda_k^{-1}}} \leq \sqrt{T} \sqrt{\sum_{k = 1}^{K \sprt{T}} \sum_{t \in \calT_k} \sprt{1 \wedge \norm{\phi_{k, t}}_{\Lambda_k^{-1}}^2}}.
\end{align*}
To proceed, we use the inequality $\sprt{x \wedge \abs{\calT_k}} \leq \frac{\abs{\calT_k}}{\log \sprt{\abs{\calT_k} + 1}} \log \sprt{1 + x}$ that is valid for all $x \geq 0$. Setting $C_k = \frac{\abs{\calT_k}}{\log \sprt{\abs{\calT_k} + 1}}$, this gives
\begin{align*}
    \sum_{k = 1}^{K \sprt{T}} \sum_{t \in \calT_k} \sprt{1 \wedge \norm{\phi_{k, t}}_{\Lambda_k^{-1}}^2} &= \sum_{k = 1}^{K \sprt{T}} \frac{1}{\abs{\calT_k}}\sum_{t \in \calT_k} \sprt{\abs{\calT_k} \wedge \abs{\calT_k} \norm{\phi_{k, t}}_{\Lambda_k^{-1}}^2} \le  \sum_{k = 1}^{K \sprt{T}} \frac{C_k}{\abs{\calT_k}}\sum_{t \in \calT_k} \log \sprt{1 + \abs{\calT_k} \norm{\phi_{k, t}}_{\Lambda_k^{-1}}^2} \\
    &\le  \max_k C_k \cdot \sum_{k = 1}^{K \sprt{T}} \frac{1}{\abs{\calT_k}}\sum_{t \in \calT_k} \log \sprt{1 + \abs{\calT_k} \norm{\phi_{k, t}}_{\Lambda_k^{-1}}^2}.
\end{align*}
The sum is handled separately in Lemma~\ref{lem:modified-ellip} stated and proved right after this proof. Putting the result together with our previous calculations, we get 
\begin{equation*}
    \sum_{k \notin \calE \sprt{T}} \sum_{t \in \calT_k} \sprt{1 \wedge \norm{\phi_{k, t}}_{\Lambda_k^{-1}}} \leq \sqrt{T} \sqrt{\max_k C_k} \sqrt{d \log \sprt{1 + \frac{B^2 H^2 T}{\lambda d}}}.
\end{equation*}
The only random quantity left in the upper-bound is the maximum over $C_k$. By concavity of the square-root function and Jensen's inequality, we get
\begin{equation*}
    \bbE \sbrk{\sum_{k \notin \calE \sprt{T}} \sum_{t \in \calT_k} \sprt{1 \wedge \norm{\phi_{k, t}}_{\Lambda_k^{-1}}}} \leq \sqrt{T} \sqrt{\bbE \sbrk{\max_k C_k}} \sqrt{d \log \sprt{1 + \frac{B^2 H^2 T}{\lambda d}}},
\end{equation*}
which we further upper bound by using Lemma~\ref{lem:max-geometric} as
\begin{equation*}
    \bbE \sbrk{\sum_{k \notin \calE \sprt{T}} \sum_{t \in \calT_k} \sprt{1 \wedge \norm{\phi_{k, t}}_{\Lambda_k^{-1}}}} \leq \sqrt{T} \sqrt{H \sprt{4 + 2 \log T}} \sqrt{d \log \sprt{1 + \frac{B^2 H^2 T}{\lambda d}}}.
\end{equation*}

We put together the two terms, and plug in the definition of $\beta$ to get
\begin{align*}
    \bbE \sbrk{\sum_{t = 1}^{T^+} \CB_t \sprt{X_t, A_t}} &\leq C_1 \sprt{T} + \sqrt{T} C_2 \sprt{T},
\end{align*}
where the two factors are defined as
\begin{align*}
    C_1 \sprt{T} &= \sprt{H \sqrt{2 \sprt{\frac{d}{2} \log \sbrk{1 + \frac{B^2 H^2 T}{\lambda d}} + \log T}} + \sqrt{\lambda} B} \frac{d B H^2}{\sqrt{\lambda} \log \sprt{2}} \log \sprt{1 + \frac{B^2 H^2 T}{\lambda d}} \\
    C_2 \sprt{T} &= \sprt{H \sqrt{2 \sprt{\frac{d}{2} \log \sbrk{1 + \frac{B^2 H^2 T}{\lambda d}} + \log T}} + \sqrt{\lambda} B} \sqrt{H \sprt{4+2\log T}} \sqrt{d \log \sprt{1 + \frac{B^2 H^2 T}{\lambda d}}}.
\end{align*}
The proof is then concluded by observing that $C_1 \sprt{T} + C_2 \sprt{T} = \calO \sprt{B H^{3/2} d \log \sprt{T}^{3/2}}$.
\qed

\begin{lemma}
\label{lem:modified-ellip}
    Following the same notations that in Section~\ref{app:elliptical}
    \begin{align*}
        \sum_{k = 1}^K \frac{1}{\abs{\calT_k}} \sum_{t \in \calT_k} \log \pa{1 + \abs{\calT_k} \norm{\phi_{k, t}}_{\Lambda_k^{-1}}^2} \leq d \log \sprt{1 + \frac{B^2 H^2 T}{\lambda d}}.
    \end{align*}
\end{lemma}

\begin{proof}
    We will follow the steps of the proof of Lemma~19.4 in \citet{LSz20}. First, notice that $\Lambda_k$ can be written as 
    \begin{equation*}
        \Lambda_{k + 1} = \Lambda_k + \sum_{t \in \calT_k} \phi_{k, t} \phi_{k, t}\transpose = \Lambda_k^{1/2} \pa{I + \sum_{t \in \calT_k} \Lambda_k^{-1/2} \phi_{k, t} \phi_{k, t}\transpose \Lambda_k^{-1/2}}\Lambda_k^{1/2}.
    \end{equation*}
    Taking the determinant of the above matrix, we get
    \begin{equation*}
        \det \pa{\Lambda_{k + 1}} = \det \pa{\Lambda_k} \det \pa{I + \sum_{t \in \calT_k} \Lambda_k^{-1/2} \phi_{k, t} \phi_{k, t}\transpose \Lambda_k^{-1/2}}.
    \end{equation*}
    Now, taking logarithms on both sides and using the concavity of $\log \det$, we obtain
    \begin{align*}
        \log \det \pa{\Lambda_{k + 1}} &= \log \det \pa{\Lambda_k} + \log \det \pa{I + \sum_{t \in \calT_k} \Lambda_k^{-1/2} \phi_{k, t} \phi_{k, t}\transpose \Lambda_k^{-1/2}} \\
        &= \log \det \pa{\Lambda_k} + \log \det \pa{\frac{1}{\abs{\calT_k}} \sum_{t \in \calT_k} \pa{I + \abs{\calT_k} \Lambda_k^{-1/2} \phi_{k, t} \phi_{k, t}\transpose \Lambda_k^{-1/2}}} \\
        &\ge \log \det \pa{\Lambda_k} + \frac{1}{\abs{\calT_k}} \sum_{t \in \calT_k} \log \det \pa{I + \abs{\calT_k} \Lambda_k^{-1/2} \phi_{k, t} \phi_{k, t}\transpose \Lambda_k^{-1/2}} \\
        &= \log \det \pa{\Lambda_k} + \frac{1}{\abs{\calT_k}} \sum_{t \in \calT_k} \log \pa{1 + \abs{\calT_k} \norm{\phi_{k, t}}_{\Lambda_k^{-1}}^2},
    \end{align*}
    where the inequality is Jensen's, and the final step follows from using the equality $\det (I + vv\transpose) = (1 + \twonorm{v}^2)$ that holds for any $v \in \real^d$. Summing up for $k$ gives
    \begin{equation*}
        \log \det \pa{\Lambda_{K \sprt{T}}} \ge \log \det \pa{\Lambda_1} + \sum_{k = 1}^K \frac{1}{\abs{\calT_k}} \sum_{t \in \calT_k} \log \pa{1 + \abs{\calT_k} \norm{\phi_{k, t}}_{\Lambda_k^{-1}}^2},
    \end{equation*}
    and furthermore by the trace-determinant inequality we have
    \begin{align*}
        \log \sprt{\frac{\det \Lambda_{K \sprt{T}}}{\det \Lambda_0}} &= \log \sprt{\frac{\det \sprt{\Lambda_{K \sprt{T}}}}{\lambda^d}} \leq d \log \sprt{\frac{\trace{\Lambda_{K \sprt{T}}}}{\lambda d}}.
    \end{align*}
    Finally, the trace can be bounded as
    \begin{equation*}
        \trace{\Lambda_{K \sprt{T}}} = \lambda d + \sum_{k = 1}^{K \sprt{T}} \sum_{t \in \calT_k} \norm{\phi_{k, t}}_2^2 \leq \lambda d + B^2 H^2 T.
    \end{equation*}
    Plugging this back into the previous inequality proves the claim.
\end{proof}

\begin{lemma}
\label{lem:bad-epochs}
The number of epochs that contain a feature vector with a norm larger than one is bounded as
    \begin{equation*}
        \abs{\calE \sprt{T}} \leq \frac{d}{\log \sprt{2}} \log \sprt{1 + \frac{B^2 H^2 T}{\lambda d}} .
    \end{equation*}
\end{lemma}
A simpler version of this statement is given as Exercise~19.3 in \citet{LSz20}, and our proof below drew inspiration from the proof of Lemma~19 in \cite{ouhamma2022bilinear}. We only have to deal with the challenge of randomized epoch schedules, which we do by similar arguments as in the proof of Lemma~\ref{lem:elliptical} above.
\begin{proof}
    Let $k \in \sbrk{K \sprt{T}}$. We define $G_0 = \lambda I$ and $G_{k + 1} = G_k + \sum_{t \in \calT_k} \phi_{k, t} \phi_{k, t}\transpose \II{t\in\calN \sprt{T}} $. We have the following decomposition:
    \begin{align*}
        G_{k + 1} &= G_k^{1 / 2} \sprt{I + \sum_{t \in \calT_k} \sprt{G_k^{- 1 / 2} \phi_{k, t}} \sprt{G_k^{- 1 / 2} \phi_{k, t}}\transpose \II{t\in\calN \sprt{T}}} G_k^{1 / 2} \\
        &= G_k^{1 / 2} \sprt{I + \sum_{t \in \calT_k \cap \calN \sprt{T}} \sprt{G_k^{- 1 / 2} \phi_{k, t}} \sprt{G_k^{- 1 / 2} \phi_{k, t}}\transpose} G_k^{1 / 2}.
    \end{align*}
    Therefore, taking the log-determinant on both sides, we obtain
    \begin{equation*}
        \log \det \sprt{G_{k + 1}} = \log \det \sprt{G_k} + \log \det \sprt{I + \sum_{t \in \calT_k \cap \calN \sprt{T}} \sprt{G_k^{- 1 / 2} \phi_{k, t}} \sprt{G_k^{- 1 / 2} \phi_{k, t}}\transpose}.
    \end{equation*}
    If $\calT_k \cap \calN \sprt{T} = \emptyset$, or equivalently if $k \notin \calE \sprt{T}$ (i.e., there is no ``bad'' state-action pair in the epoch $k$), the second term in the right-hand side is zero. Hence, summing over $k \in \sbrk{\calE \sprt{T}}$, we get
    \begin{equation}
        \log \det \sprt{G_{K \sprt{T}}} = \log \det \sprt{G_0} + \sum_{k \in \calE \sprt{T}} \log \det \sprt{I + \sum_{t \in \calT_k \cap \calN \sprt{T}} \sprt{G_k^{- 1 / 2} \phi_{k, t}} \sprt{G_k^{- 1 / 2} \phi_{k, t}}\transpose}.
    \end{equation}
    Using the concavity of $\log \det$, Jensen's inequality gives us
    \begin{align*}
        \log \det \sprt{G_{K \sprt{T}}} &\geq \log \det \sprt{G_0} \\
        &\qquad + \sum_{k \in \calE \sprt{T}} \frac{1}{\abs{\calT_k \cap \calN \sprt{T}}} \sum_{t \in \calT_k \cap \calN \sprt{T}} \log \det \sprt{I + \abs{\calT_k \cap \calN \sprt{T}} \sprt{G_k^{- 1 / 2} \phi_{k, t}} \sprt{G_k^{- 1 / 2} \phi_{k, t}}\transpose} \\
        &= \log \det \sprt{G_0} + \sum_{k \in \calE \sprt{T}} \frac{1}{\abs{\calT_k \cap \calN \sprt{T}}} \sum_{t \in \calT_k \cap \calN \sprt{T}} \log \sprt{1 + \abs{\calT_k \cap \calN \sprt{T}} \norm{\phi_{k, t}}_{G_k^{-1}}^2},
    \end{align*}
    where the equality follows from the fact that $\det (I + vv\transpose) = (1 + \twonorm{v}^2)$ that holds for any $v \in \real^d$. Then, we notice that $G_k^{-1} \succeq \Lambda_k^{-1}$, and thus we can further bound this expression as
    \begin{equation*}
        \log \det \sprt{G_{K \sprt{T}}} \geq \log \det \sprt{G_0} + \sum_{k \in \calE \sprt{T}} \frac{1}{\abs{\calT_k \cap \calN \sprt{T}}} \sum_{t \in \calT_k \cap \calN \sprt{T}} \log \sprt{1 + \abs{\calT_k \cap \calN \sprt{T}} \norm{\phi_{k, t}}_{\Lambda_k^{-1}}^2}.
    \end{equation*}
    For $k \in \calE \sprt{T}$, $t \in \calT_k \cap \calN \sprt{T}$, we have $\abs{\calT_k \cap \calN \sprt{T}} \geq 1$, and $\norm{\phi_{k, t}}_{\Lambda_k^{-1}} \geq 1$ by definition of $\calN \sprt{T}$. This implies that
    \begin{align*}
        \log \det \sprt{G_{K \sprt{T}}} &\geq \log \det \sprt{G_0} + \sum_{k \in \calE \sprt{T}} \frac{1}{\abs{\calT_k \cap \calN \sprt{T}}} \sum_{t \in \calT_k \cap \calN \sprt{T}} \log \sprt{2} \\
        &\geq \log \det \sprt{G_0} + \log \sprt{2} \abs{\calE \sprt{T}}.
    \end{align*}
    Thus, we have
    \begin{align*}
        \abs{\calE \sprt{T}} &\leq \frac{1}{\log \sprt{2}} \log \sprt{\frac{\det \sprt{G_{K \sprt{T}}}}{\det \sprt{G_0}}} \\
        &= \frac{1}{\log \sprt{2}} \log \sprt{\frac{\det \sprt{G_{K \sprt{T}}}}{\lambda^d}} & \text{(by the definition of  $G_1$)} \\
        &\leq \frac{d}{\log \sprt{2}} \log \sprt{\frac{\trace{G_{K \sprt{T}}}}{\lambda d}}. & \text{(by the trace-determinant inequality)}
    \end{align*}
    Finally, the trace can be bounded as
    \begin{equation*}
        \trace{G_{K \sprt{T}}} = \lambda d + \sum_{k \in \calE \sprt{T}} \sum_{t \in \calT_k \cap \calN \sprt{T}} \norm{\phi_{k, t}}_2^2 \mathbbm{1}_{\calN \sprt{T}} \sprt{t} \leq \lambda d + B^2 H^2 T.
    \end{equation*}
    The proof is concluded by putting this bound together with the previous inequality.
\end{proof}

\begin{lemma}
\label{lem:max-geometric}
The random variables $\scbrk{C_k}_k$, defined for all $k$ by $C_k = \frac{\abs{\calT_k}}{\log \sprt{1 + \abs{\calT_k}}}$ where $\abs{\calT_k}$ is a geometric random variable with parameter $1 - \gamma$, satify the following
    \begin{equation*}
        \EE{\max_k C_k} \leq \frac{4 + 2 \log T}{1 - \gamma}.
    \end{equation*}
\end{lemma}
\begin{proof}
First, notice that $\log \sprt{1 + \abs{\calT_k}} \geq \log 2$ so that $C_k = \frac{\abs{\calT_k}}{\log \sprt{1 + \abs{\calT_k}}} \le \frac{\abs{\calT_k}}{\log 2}$. Next, using the fact that the number of 
epochs is at most $T$, and observing that each $\abs{\calT_k}$ is geometrically distributed with parameter $1 - \gamma$, we can bound $\max_k \abs{\calT_k}$ by a maximum over $T$ independent geometric random variables $Z_1, \dots, Z_T$ with 
parameter $1 - \gamma$:
    \begin{align*}
        \EE{\max_k \abs{\calT_k}} &\leq \EE{\max_{j \in \sbrk{T}} Z_j}\\
        &= \sum_{i=0}^\infty \PP{\max_{j \in \sbrk{T}} Z_j > i} &\mbox{(since each $Z_i$ is nonnegative)}\\
        &\leq {k + T \sum_{i = k}^\infty \PP{Z_1 > i}} & \mbox{(upper bounding the first $k>0$ terms by $1$)} \\
        &= {k + T \sum_{i = k}^\infty \gamma^i} & \mbox{(using that $Z_1$ is geometric with parameter $1 - \gamma$)}\\
        &= {k + \frac{T \gamma^k}{1 - \gamma}},
    \end{align*}
where we have used the formula for the geometric sum in the last step. Now, setting $k = \left\lceil\frac{\log T}{1 - \gamma}\right\rceil$, we get
\begin{align*}
    \EE{\max_k \abs{\calT_k}} &= {k + T \frac{\gamma^k}{1 - \gamma}} \le \frac{1 + \log T}{1 - \gamma} + \frac{T \exp \pa{\frac{\log \gamma}{1 - \gamma} \cdot \log T}}{1-\gamma} \\
    &\le \frac{1 + \log T}{1 - \gamma} + \frac{T \exp \pa{- \log T}}{1 - \gamma} = \frac{2 + \log T}{1 - \gamma},
\end{align*}
where in the second line we have used the inequality $\frac{\log \gamma}{1 - \gamma} \le - 1$ that holds for all 
$\gamma \in \sprt{0, 1}$. The proof is concluded by using that $\log 2 > \frac 12$ and combining the above bound with the bound relating $C_k$ to $\abs{\calT_k}$.
\end{proof}

%% file: C_algorithms.tex

\section{Applications: Algorithm Specifications}

For clarity, we provide the complete algorithms in the context of tabular and linear mixture MDPs. The highlighted parts correspond to the instantiations of the functions \texttt{TRANSITION-ESTIMATE}, \texttt{BONUS}, and \texttt{ADD} from Algorithm~\ref{alg:ravi-ucb}.

    \subsection{Tabular MDPs}\label{app:ravi-ucb-tabular}

In tabular MDPs, we use the maximum likelihood estimates to compute $\wh{P}_k$ and the classical count-based bonuses. Therefore, we only need to store and update the counts $N_t$ and $N_t'$ when interacting with the environment.

\begin{algorithm}[h]
    \caption{\MainAlg for tabular MDPs.}
    \begin{algorithmic}
    \label{alg:ravi-ucb-tabular}
        \STATE {\bfseries Inputs:} Horizon $T$, learning rate $\eta > 0$, confidence parameter $\beta > 0$, value $V_0$, policy $\pi_0$.
        \STATE {\bfseries Initialize:} $t = 1$, $N_1' = 0$, $N_1 = 1$, $Q_1 = E V_0$.
        \FOR{$k = 1, \dots$}
        \STATE $T_k = t$.
        \STATE \textcolor{blue}{$\wh{P}_{k} \sprt{x' \given x, a} = N_{T_k}' \sprt{x, a, x'} / N_{T_k} \sprt{x, a}$.}
        \STATE $V_k \sprt{x} = \frac1\eta \log \sprt{\sum_a \pi_{k - 1} \sprt{a \given x} e^{\eta Q_k \sprt{x, a}}}$.
        \STATE $\pi_k \sprt{a \given x} = \pi_{k - 1} \sprt{a \given x} e^{\eta \sprt{Q_k \sprt{x, a} - V_k \sprt{x}}}$.
        \STATE \textcolor{blue}{$\CB_k \sprt{x, a} = \beta / \sqrt{N_{T_k} \sprt{x, a}}$.}
        \STATE $Q_{k + 1} = \Pi_H\! \sbrk{r + \CB_k + \gamma \wh{P}_k V_k}$.
        \REPEAT
        \STATE Play $a_t \sim \pi_k \sprt{\cdot \given x_t}$, and observe $x_{t + 1}$.
        \STATE \textcolor{blue}{Update $N_{t + 1}' \sprt{x_t, a_t, x_{t + 1}} = N_t' \sprt{x_t, a_t, x_{t + 1}} + 1$, and $N_{t + 1} \sprt{x_t, a_t} = N_t \sprt{x_t, a_t} + 1$.}
        \STATE $t = t + 1$.
        \STATE With probability $1 - \gamma$, \reset to initial distribution: $x_t \sim \nu_0$ and \textbf{break}.
        \UNTIL{$t = T$}
        \ENDFOR
    \end{algorithmic}
\end{algorithm}

    \subsection{Linear Mixture MDPs}\label{app:ravi-ucb-mixture}

In linear mixture MDPs, $\wh{P}_k$ is computed via a least-squares regression, and we use elliptical bonuses for $\CB_k$. Thus, we only need to store and update the empirical covariance matrix $\Lambda_t$ and the vector $b_t$ when interacting with the environment.

\begin{algorithm}[h]
    \caption{\MainAlg for linear mixture MDPs.}
    \begin{algorithmic}
    \label{alg:ravi-ucb-linear-mixture}
        \STATE {\bfseries Inputs:} Horizon $T$, learning rate $\eta > 0$, confidence parameter $\beta > 0$, regularization parameter $\lambda$, value $V_0$, policy $\pi_0$.
        \STATE {\bfseries Initialize:} $t = 1$, $\Lambda_1 = \lambda I$, $b_1 = 0$, $Q_1 = E V_0$.
        \FOR{$k = 1, \dots$}
        \STATE $T_k = t$.
        \STATE \textcolor{blue}{$\wh{\theta}_k = \Lambda_{T_k}^{-1} b_{T_k}$.}
        \STATE \textcolor{blue}{$\wh{P}_{k} = \sum_i \htheta_{k, i} \psi_i$.}
        \STATE $V_k \sprt{x} = \frac1\eta \log \sprt{\sum_a \pi_{k - 1} \sprt{a \given x} e^{\eta Q_k \sprt{x, a}}}$.
        \STATE $\pi_k \sprt{a \given x} = \pi_{k - 1} \sprt{a \given x} e^{\eta \sprt{Q_k \sprt{x, a} - V_k \sprt{x}}}$.
        \STATE \textcolor{blue}{$\phi_k \sprt{x, a} = \sum_{x'} \psi \sprt{x, a, x'} V_k \sprt{x'}$.}
        \STATE \textcolor{blue}{$\CB_k \sprt{x, a} = \beta \norm{\phi_k \sprt{x, a}}_{\Lambda_{T_k}^{-1}}$.}
        \STATE $Q_{k + 1} = \Pi_H\! \sbrk{r + \CB_k + \gamma \wh{P}_k V_k}$.
        \REPEAT
        \STATE Play $a_t \sim \pi_k \sprt{\cdot \given x_t}$, and observe $x_{t + 1}$.
        \STATE \textcolor{blue}{Update $\Lambda_{t + 1} = \Lambda_t + \phi_k \sprt{x_t, a_t} \phi_k \sprt{x_t, a_t}\transpose$, and $b_{t + 1} = b_t + \phi_k \sprt{x_t, a_t} V_k \sprt{x_{t + 1}}$.}
        \STATE $t = t + 1$.
        \STATE With probability $1 - \gamma$, \reset to initial distribution: $x_t \sim \nu_0$ and \textbf{break}.
        \UNTIL{$t = T$}
        \ENDFOR
    \end{algorithmic}
\end{algorithm}

%% file: B_standard_results.tex

\section{Standard Results}

\subsection{Softmax Policies and Value Functions}

In this section, we recall a range of standard facts relating the softmax policies our algorithm uses and the associated value functions. These can be found in numerous papers, textbooks, and lecture notes---for concreteness, see Section~28.1 in \citealp{LSz20} as an example.

\begin{lemma}
\label{lem:md-kl-update}
    Let $\scbrk{V_k}_{k \in \sbrk{K}}$, $\scbrk{\pi_k}_{k \in \sbrk{K}}$, and $\scbrk{Q_k}_{k \in \sbrk{K}}$ be the sequences of functions defined in Algorithm~\ref{alg:ravi-ucb}. Then, the following equalities are satisfied for all $k\in[K]$ and $x\in\calX$:
    \begin{align*}
        V_k \sprt{x} &= \max_{p \in \Delta \sprt{\calA}} \scbrk{\inp{p, Q_k \sprt{x, \cdot}} - \frac1\eta \KL \sprt{p \| \pi_{k - 1} \sprt{\cdot \given x}}} \\
        \pi_k \sprt{\cdot \given x} &= \argmax_{p \in \Delta \sprt{\calA}} \scbrk{\inp{p, Q_k \sprt{x, \cdot}} - \frac1\eta \KL \sprt{p \| \pi_{k - 1} \sprt{\cdot \given x}}}.
    \end{align*}
    Furthermore, for all $k\in[K]$ and $x\in\calX$, we have
    \begin{equation*}
        \sum_{i = 1}^k V_i \sprt{x} = \max_{p \in \Delta \sprt{\calA}} \scbrk{\inp{p, \sum_{i = 1}^k Q_i \sprt{x, \cdot}} - \frac1\eta \KL \sprt{p \| \pi_0 \sprt{\cdot \given x}}}.
    \end{equation*}
\end{lemma}

\begin{proof}
    First, we show that the maximum indeed takes the form claimed in the main paper and that the maximizer is given by a softmax policy. For simplicity, we drop the indices for now and consider the optimization problem
    \begin{equation*}
        \sup_{p \in \Delta \sprt{\calA}} \scbrk{\inp{p, Q} - \frac1\eta \KL \sprt{p \| p'}},
    \end{equation*}
    where $Q \in \bbR^\calA$, and $p' \in \Delta \sprt{\calA}$. As the probability simplex is compact and $\sprt{p \mapsto \inp{p, Q} - \frac1\eta \KL \sprt{p \| p'}}$ is continuous, the supremum is attained at some $p^* \in \Delta \sprt{\calA}$. The Lagrangian function of this optimization problem is given for all $p \in \bbR_+^\calA$ and $ \alpha \in \bbR$ as
    \begin{equation*}
        \calL \sprt{p, \alpha} = \inp{p, Q} - \frac1\eta \KL \sprt{p \| p'} + \alpha \sprt{\inp{p, \mathbf{1}} - 1}.
    \end{equation*}
    Its partial derivative with respect to the primal variable $p(a)$ is
    \begin{equation*}
        \frac{\partial{\calL \sprt{p, \alpha}}}{\partial p(a)} = Q(a) - \frac1\eta \sprt{\log \sprt{\frac{p(a)}{p'(a)}} + 1} + \alpha.
    \end{equation*}
    Setting it to zero gives us the expression
    \begin{equation*}
        p^*(a) = p'(a) \exp \bpa{\eta \sprt{Q(a) + \alpha} - 1}.
    \end{equation*}
    Then, we use the constraint on $p^*$ to find the value of $\alpha$. In particular, $\inp{p^*, \mathbf{1}} = 1$ implies
    \begin{equation*}
        \sum_{a \in \calA} p' \sprt{a} \exp \pa{\eta Q \sprt{a}} = \exp \pa{1 - \eta \alpha},
    \end{equation*}
    from which we deduce that
    \begin{equation*}
        \alpha = \frac1\eta \sprt{1 - \log \sprt{\sum_{a \in \calA} p' \sprt{a} \exp \pa{\eta Q \sprt{a}}}}.
    \end{equation*}
    Denoting $V^* = \frac1\eta \log \sprt{\sum_{a \in \calA} p' \sprt{a} \exp \sbrk{\eta Q \sprt{a}}}$, we plug back the expression of $\alpha$ into $p^*$:
    \begin{equation*}
        p^*(a) = p'(a) \exp \bpa{\eta \sprt{Q(a) - V^*}}.
    \end{equation*}
    From this, we can directly express the relative entropy between $p^*$ and $p'$ as
    \begin{equation*}
        \KL \sprt{p^* \| p'} = \sum_a p^*(a) \log \frac{p^*(a)}{p'(a)} = \sum_a p^*(a) \bpa{Q(a) - V^*} = \iprod{p^*}{Q - V^*\mathbf{1}},
    \end{equation*}
    so that we can write
    \begin{equation*}
        \inp{p^*, Q} - \frac1\eta \KL \sprt{p^* \| p'} = \inp{p^*, Q} - \inp{p^*, Q - V^* \mathbf{1}} = V^*.
    \end{equation*}
    The first statement of the lemma then follows from applying this result to $Q = Q_k \sprt{x, \cdot}$ and $p' = \pi_{k - 1} \sprt{\cdot \given x}$, for $k \in \sbrk{K}$, $x \in \calX$. That is, for any state-action pair $\sprt{x, a} \in \calX \times \calA$ and $k \geq 1$, denoting the maximum $V_k$ and the maximizer $\pi_k$, we have that the following expressions are equivalent to the ones given in the statement of the lemma:
    \begin{align*}
        V_k \sprt{x} &= \frac1\eta \log \sprt{\sum_{a \in \calA} \pi_{k-1} \sprt{a \given x} e^{\eta Q_k \sprt{x, a}}}, \\
        \pi_k \sprt{a \given x} &= \pi_{k-1} \sprt{a \given x} \exp \sprt{\eta \sbrk{Q_k \sprt{x, a} - V_k \sprt{x}}}.
    \end{align*}
    For the second statement, we start by denoting $\bar V_k = \sum_{i = 1}^k V_i$ and $\bar Q_k = \sum_{i = 1}^k Q_i$, for $k \geq 1$, and show by induction that, for $x \in \calX$, the following holds
    \begin{equation*}
        \pi_k \sprt{\cdot \given x} = \pi_0 \sprt{\cdot \given x} \exp \sprt{\eta \sbrk{\bar Q_k \sprt{x, \cdot} - \bar V_k \sprt{x} \mathbf{1}}}.
    \end{equation*}
    Let $x \in \calX$. The case $k = 1$ follows immediately from the previous statement with $Q = Q_1 \sprt{x, \cdot}$ and $p' = \pi_0 \sprt{\cdot \given x}$. Assume the previous equation holds at $k$. Using the first statement with $Q = Q_{k + 1} \sprt{x, \cdot}$ and $p' = \pi_k \sprt{\cdot \given x}$ we have, for $a \in \calA$, $\pi_{k + 1} \sprt{a \given x} = \pi_k \sprt{a \given x} e^{\eta \sbrk{Q_{k + 1} \sprt{x, a} - V_{k + 1} \sprt{x}}}$. Applying the inductive hypothesis, it gives
    \begin{align*}
        \pi_{k + 1} \sprt{a \given x} &= \pi_0 \sprt{a \given x} \exp \sprt{\eta \sbrk{\bar Q_k \sprt{x, a} - \bar V_k \sprt{x}}} \exp \sprt{\eta \sbrk{Q_{k + 1} \sprt{x, a} - V_{k + 1} \sprt{x}}} \\
        &= \pi_0 \sprt{a \given x} \exp \sprt{\eta \sbrk{\bar Q_{k + 1} \sprt{x, a} - \bar V_{k + 1} \sprt{x}}},
    \end{align*}
    which finishes the induction. Then, we move on to the actual statement. We have
    \begin{align*}
        V_k \sprt{x} &= \frac1\eta \log \sprt{\sum_{a \in \calA} \pi_{k - 1} \sprt{a \given x} e^{\eta Q_k \sprt{x, a}}} & \text{(by the first statement)} \\
        &= \frac1\eta \log \sprt{\sum_{a \in \calA} \pi_0 \sprt{a \given x} e^{\eta \sprt{Q_k \sprt{x, a} + \bar Q_{k - 1} \sprt{x, a} - \bar V_{k - 1} \sprt{x}}}} & \text{(by induction)} \\
        &= \frac1\eta \log \sprt{\sum_{a \in \calA} \pi_0 \sprt{a \given x} e^{\eta \sprt{\bar Q_k \sprt{x, a}}}} - \bar V_{k - 1} \sprt{x}.
    \end{align*}
    Therefore, by definition of $\bar V_{k - 1}$,
    \begin{align*}
        \sum_{i = 1}^k V_i \sprt{x} &= \frac1\eta \log \sprt{\sum_{a \in \calA} \pi_0 \sprt{a \given x} e^{\eta \sprt{\bar Q_k \sprt{x, a}}}} \\
        &= \max_{p \in \Delta \sprt{\calA}} \scbrk{\inp{p, \sum_{i = 1}^k Q_i \sprt{x, \cdot}} - \frac1\eta \KL \sprt{p \| \pi_0 \sprt{\cdot \given x}}},
    \end{align*}
    which concludes the proof.
\end{proof}

\subsection{A Self-Normalized Tail Inequality}\label{app:concentration}

\begin{theorem}[Theorem~14.7 in \citet{dlPLS09}, Theorem~2 in \citet{abbasi2011improved}]
\label{thm:self-concentration}
    Let $\scbrk{\eta_t}_{t=1}^\infty$ be a real-valued stochastic process with corresponding filtration $\scbrk{\calF_t}_{t=0}^\infty$. Let $\eta_t | \calF_{t-1}$ be zero-mean and $\sigma$-subGaussian; \ie\ $\bbE \sbrk{\eta_t \given \calF_{t-1}} = 0$, and
    \begin{equation*}
        \forall \lambda \in \bbR, \bbE \sbrk{e^{\lambda \eta_t} \given \calF_{t-1}} \leq e^{\frac{\lambda^2 \sigma^2}{2}}.
    \end{equation*}
    Let $\scbrk{\phi_t}_{t = 0}^\infty$ be an $\bbR^d$-valued stochastic process where $\phi_t$ is $\calF_{t-1}$-measurable. Assume $\Lambda_0$ is a $d \times d$ positive definite matrix, and let $\Lambda_t = \Lambda_0 + \sum_{s=1}^t \phi_s \phi_s\transpose$. Then, for any $\delta > 0$, with probability at least $1 - \delta$, we have for all $t \geq 0$,
    \begin{equation*}
        \norm{\sum_{s=1}^t \phi_s \eta_s}_{\Lambda_t^{-1}}^2 \leq 2 \sigma^2 \log \sbrk{\frac{\det \sprt{\Lambda_t}^{1 / 2} \det \sprt{\Lambda_0}^{- 1 / 2}}{\delta}}.
    \end{equation*}

\end{theorem}